\documentclass[runningheads]{llncs}

\usepackage{graphicx}
\usepackage[utf8]{inputenc}

\usepackage{enumitem} %

\usepackage{tikz}
\usepackage[inference]{semantic}

\usetikzlibrary{automata}
\usetikzlibrary{arrows}
\usetikzlibrary{positioning}
\usetikzlibrary{calc}
\usetikzlibrary{fadings}

\tikzstyle{automaton}=[
  semithick,shorten >=1pt,>=stealth',
  node distance=1cm,
  initial text=,
  every initial by arrow/.style={every node/.style={inner sep=0pt}},
  every state/.style={minimum size=7.5mm,fill=white}
]
\tikzstyle{smallautomaton}=[
  automaton,
  node distance=5mm,
  every state/.style={minimum size=4mm,fill=white,inner sep=1pt}
]

\tikzstyle{statename}=[
  label distance=2pt,
  fill=yellow!30!white,
  rounded corners=1mm,inner sep=2pt
]

\tikzstyle{accset}=[
  fill=blue!50!black,draw=white,text=white,thin,
  circle,inner sep=.9pt,anchor=center,font=\bfseries\sffamily\tiny
]

\tikzstyle{loop 30}=[in=15,out=45,loop,right]
\tikzstyle{loop 45}=[in=30,out=60,loop,right]
\tikzstyle{loop 135}=[in=120,out=150,loop,left]
\tikzstyle{loop -30}=[in=-45,out=-15,loop,right]
\tikzstyle{loop -45}=[in=-60,out=-30,loop,right]
\tikzstyle{loop -60}=[in=-75,out=-45,loop,right]
\tikzstyle{loop -120}=[in=-135,out=-105,loop,left]
\tikzstyle{loop -135}=[in=-150,out=-120,loop,left]

\definecolor{myYellow}{HTML}{FFCC00}
\definecolor{myOrange}{HTML}{FF9900}
\definecolor{myBlue}{HTML}{99CCFF}
\definecolor{yourBlue}{rgb}{0.6, 0.73, 0.89}

\tikzstyle{knoten}=[state,very thick,minimum size=1cm,font=\strut]
\tikzstyle{kanten}=[->,very thick,>=stealth']

\tikzstyle{myarrow} = [->, semithick, >=stealth, shorten >=1]
\tikzstyle{mybackarrow} = [<-, semithick, >=stealth, shorten <=1]
\tikzstyle{long myarrow} = [myarrow, shorten <= -2] %
\tikzstyle{light arrow} = [long myarrow, black!40]
\tikzstyle{fade out arrow} = [myarrow, path fading=south]
\tikzstyle{basic state} = [fill=black!10,  draw=black!10]
\tikzstyle{state} = [basic state, circle, inner sep=1pt, minimum size=7mm]
\tikzstyle{state'} = [basic state, circle, inner sep=1pt, minimum size=5.5mm, node font=\small, fill=black!5]
\tikzstyle{MCstate} = [basic state, rectangle, minimum size=7mm, rounded corners=3pt]
\tikzstyle{tiny state} = [basic state, circle, minimum size=12pt, inner sep=2pt]
\tikzstyle{label} = [font=\small]

\tikzstyle{mynode} = [rectangle, rounded corners=2, minimum height=6mm,inner ysep=0pt, minimum width=14mm]

\usepackage{todonotes}
\usepackage[normalem]{ulem} %

\usepackage{amsmath,amssymb,eurosym}
\usepackage{stmaryrd}
\usepackage{xspace}
\usepackage{booktabs} %
\usepackage{hyperref}
\usepackage{mathtools}

\usepackage{cleveref}
\usepackage{nicefrac}

\usepackage{subcaption}

\makeatletter
\def\orcidID#1{\smash{\href{http://orcid.org/#1}{\protect\raisebox{-1.25pt}{\protect\includegraphics{ORCID_Color.eps}}}}}
\makeatother

\renewcommand{\subsubsection}[1]{\paragraph{#1.}}

\newcommand{\CPTable}{\mathrm{Pr}}

\newcommand{\distBN}{\ensuremath{\mathit{dist}_\BN}}
\newcommand{\BNsem}[1]{\ensuremath{\distBN(#1)}}

\newcommand{\dissectedBN}[3]{\textit{Dissect}(#1,#2,#3)}

\newcommand{\CutsetMC}[2]{\textit{CMC}(#1,#2)}

\newcommand{\Indep}{\ensuremath{\mathit{Indep}}}
\newcommand{\dsep}{\ensuremath{\mathit{d\text{-}sep}}}

\newcommand{\Extend}{\ensuremath{\mathit{Extend}}}
\newcommand{\Next}{\ensuremath{\mathit{Next}}}
\newcommand{\nextdisBN}[3]{\ensuremath{\Next(#1,#2,#3)}} %

\newcommand{\LongRunFreq}{\mathrm{lrf}}

\newcommand{\nocycle}{\ensuremath{{\not\circlearrowright}}}

\newcommand{\cB}{\ensuremath{\mathcal{B}}\xspace}
\newcommand{\cC}{\ensuremath{\mathcal{C}}\xspace}
\newcommand{\cD}{\ensuremath{\mathcal{D}}\xspace}
\newcommand{\cE}{\ensuremath{\mathcal{E}}\xspace}

\newcommand{\cG}{\ensuremath{\mathcal{G}}\xspace}

\newcommand{\cI}{\ensuremath{\mathcal{I}}\xspace}

\newcommand{\cM}{\ensuremath{\mathcal{M}}\xspace}

\newcommand{\cP}{\ensuremath{\mathcal{P}}\xspace}

\newcommand{\cS}{\ensuremath{\mathcal{S}}\xspace}

\newcommand{\cU}{\ensuremath{\mathcal{U}}\xspace}
\newcommand{\cV}{\ensuremath{\mathcal{V}}\xspace}
\newcommand{\cW}{\ensuremath{\mathcal{W}}\xspace}
\newcommand{\cX}{\ensuremath{\mathcal{X}}\xspace}
\newcommand{\cY}{\ensuremath{\mathcal{Y}}\xspace}
\newcommand{\cZ}{\ensuremath{\mathcal{Z}}\xspace}

\newcommand{\boldssf}[1]{{\normalfont\textbf{#1}}\xspace}
\newcommand{\bP}{\boldssf{P}}

\makeatletter
\newcommand{\oset}[3][-.2ex]{%
  \mathrel{\mathop{#3}\limits^{
    \vbox to#1{\kern-2\ex@
    \hbox{$\scriptstyle#2$}\vss}}}}
\makeatother

\newcommand{\Pre}{\ensuremath{\mathit{Pre}}}
\newcommand{\Post}{\ensuremath{\mathit{Post}}}

\newcommand{\PostS}{\ensuremath{\Post^*}}

\newcommand{\Init}{\ensuremath{\mathit{Init}}}
\newcommand{\initCl}{\ensuremath{\mathit{Close}}}

\newcommand{\Dirac}{\mathit{Dirac}}

\newcommand{\Asg}{\ensuremath{\mathit{Asg}}}

\newcommand{\Dist}{\mathit{Dist}} %
\newcommand{\dist}{\ensuremath{\mu}\xspace}  %
\newcommand{\combi}{\ensuremath{\otimes}} %

\newcommand{\semshape}[1]{\text{\textnormal{{\scshape #1}}}\xspace}
\newcommand{\CPT}{\semshape{Cpt}}
\newcommand{\CPTI}{{\CPT{}\text{-}\cI}\xspace}
\newcommand{\wCPT}{\semshape{wCpt}}
\newcommand{\wCPTI}{{\wCPT{}\text{-}\cI}\xspace}

\newcommand{\BN}{\semshape{BN}}

\newcommand{\scGBN}{\semshape{scGBN}}
\newcommand{\Cut}{\semshape{MC}}
\newcommand{\CutC}{{\Cut{\text{-}}\cC}\xspace}
\newcommand{\Lim}{\semshape{Lim}}
\newcommand{\LimC}{{\Lim{\text{-}}\cC}\xspace}

\newcommand{\LimAvg}{\semshape{LimAvg}}
\newcommand{\LimAvgC}{{\LimAvg{\text{-}}\cC}\xspace}

\newcommand{\init}{\iota}

\newcommand{\boolvalstyle}[1]{\text{{\normalfont\sffamily \resizebox{0.8\width}{!}{#1}}}\xspace}
\newcommand{\T}{\boolvalstyle{T}}
\newcommand{\F}{\boolvalstyle{F}}

\newcommand{\Nat}{\mathbb{N}}

\newcommand{\Bool}{\mathbb{B}}

\newcommand{\ssep}{\ {:}\ }

\newcommand{\semantics}[1]{\llbracket #1 \rrbracket}
\newcommand{\tuple}[1]{\langle #1 \rangle}
\newcommand{\bigtuple}[1]{\bigl\langle #1 \bigr\rangle}

\newcommand{\CUTsem}[2]{\ensuremath{\semantics{#1}_{ \Cut\text{-}#2} }\xspace}

\makeatletter
\let\save@mathaccent\mathaccent
\newcommand*\if@single[3]{%
  \setbox0\hbox{${\mathaccent"0362{#1}}^H$}%
  \setbox2\hbox{${\mathaccent"0362{\kern0pt#1}}^H$}%
  \ifdim\ht0=\ht2 #3\else #2\fi
  }
\newcommand*\rel@kern[1]{\kern#1\dimexpr\macc@kerna}
\newcommand*\widebar[1]{\@ifnextchar^{{\wide@bar{#1}{0}}}{\wide@bar{#1}{1}}}
\newcommand*\wide@bar[2]{\if@single{#1}{\wide@bar@{#1}{#2}{1}}{\wide@bar@{#1}{#2}{2}}}
\newcommand*\wide@bar@[3]{%
  \begingroup
  \def\mathaccent##1##2{%
    \let\mathaccent\save@mathaccent
    \if#32 \let\macc@nucleus\first@char \fi
    \setbox\z@\hbox{$\macc@style{\macc@nucleus}_{}$}%
    \setbox\tw@\hbox{$\macc@style{\macc@nucleus}{}_{}$}%
    \dimen@\wd\tw@
    \advance\dimen@-\wd\z@
    \divide\dimen@ 3
    \@tempdima\wd\tw@
    \advance\@tempdima-\scriptspace
    \divide\@tempdima 10
    \advance\dimen@-\@tempdima
    \ifdim\dimen@>\z@ \dimen@0pt\fi
    \rel@kern{0.6}\kern-\dimen@
    \if#31
      \overline{\rel@kern{-0.6}\kern\dimen@\macc@nucleus\rel@kern{0.4}\kern\dimen@}%
      \advance\dimen@0.4\dimexpr\macc@kerna
      \let\final@kern#2%
      \ifdim\dimen@<\z@ \let\final@kern1\fi
      \if\final@kern1 \kern-\dimen@\fi
    \else
      \overline{\rel@kern{-0.6}\kern\dimen@#1}%
    \fi
  }%
  \macc@depth\@ne
  \let\math@bgroup\@empty \let\math@egroup\macc@set@skewchar
  \mathsurround\z@ \frozen@everymath{\mathgroup\macc@group\relax}%
  \macc@set@skewchar\relax
  \let\mathaccentV\macc@nested@a
  \if#31
    \macc@nested@a\relax111{#1}%
  \else
    \def\gobble@till@marker##1\endmarker{}%
    \futurelet\first@char\gobble@till@marker#1\endmarker
    \ifcat\noexpand\first@char A\else
      \def\first@char{}%
    \fi
    \macc@nested@a\relax111{\first@char}%
  \fi
  \endgroup
}
\makeatother
\newcommand{\mnot}[1]{\ensuremath{\widebar{#1}}}

\begin{document}
\title{On the Foundations of Cycles\\ in Bayesian Networks\thanks{This work was partially supported by the DFG in projects TRR 248 (CPEC, see {\footnotesize\url{https://perspicuous-computing.science}}, project ID 389792660) and EXC 2050/1 (CeTI, project ID 390696704, as part of Germany's Excellence Strategy), and the Key-Area Research and Development Program Grant 2018B010107004 of Guangdong Province.}
}
\author{
Christel Baier\inst{1} %
\and
Clemens Dubslaff\inst{1} %
\and
Holger Hermanns\inst{2,3} %
\and
Nikolai Käfer\inst{1} %
}
\authorrunning{C. Baier et al.}
\institute{TU Dresden, Dresden, Germany \and Saarland University, Saarbrücken, Germany \and Institute of Intelligent Software, Guangzhou, China
}
\maketitle              %
\begin{abstract}
	Bayesian networks (BNs) are a probabilistic graphical model widely used for representing expert knowledge and reasoning under uncertainty.
	Traditionally, they are based on directed acyclic graphs that capture dependencies between random variables. 
	However, directed cycles can naturally arise when 
	cross-dependencies between random variables exist, e.g., for modeling feedback loops.
	Existing methods to deal with such cross-dependencies usually rely on reductions to 
	BNs without cycles.
	These approaches are fragile to generalize, since their justifications are 
	intermingled with additional knowledge about the application context.
	In this paper, we present a foundational study regarding semantics for cyclic BNs
	that are generic and conservatively extend the cycle-free setting.
	First, we propose constraint-based semantics that specify requirements for full joint distributions over a BN to be consistent with the local conditional probabilities and independencies.
	Second, two kinds of limit semantics that formalize infinite unfolding approaches are introduced and shown to be computable by a Markov chain construction.
\end{abstract}

\section{Introduction}

A \emph{Bayesian network} (BN) is a probabilistic graphical model representing
a set of random variables and their conditional dependencies. 
BNs are ubiquitous across many fields where reasoning under uncertainties is 
of interest~\cite{JenNie07a}. 
Specifically, a BN is a directed acyclic graph with the random variables as nodes and edges manifesting conditional dependencies, quantified by \emph{conditional probability tables} (CPTs). 
The probability of any random variable can then be deduced by the CPT entries along all 
its predecessors. Here, these probabilities are independent of all variables that are no 
(direct or transitive) predecessors in the graph. Acyclicity is hence
crucial and commonly assumed to be rooted in some sort of causality~\cite{Pearl09}. 
A classical use of BNs is in expert systems~\cite{Pea90a} where BNs aggregate statistical 
data obtained by several independent studies. 
In the medical domain, e.g., they can capture the correlation of certain symptoms, diseases, and human factors~\cite{JenAndKja87a,LauSpi88a,RobOwe97a}.

\begin{figure}[t]
	\centering
	\begin{tikzpicture}
		\node (A) at (0,0) [state] {$X$};
		\node (cpt A) at (-2,0) {$\setlength\arraycolsep{5pt}
			\begin{array}{c|cc}
				Y  & X\!{=}\T & X\!{=}\F \\
			\hline
				\F & s_1      & 1-s_1    \\
				\T & s_2      & 1-s_2
			\end{array}$};
		\node (B) at (1.5,0) [state] {$Y$}
			edge [mybackarrow, bend right] (A)
			edge [myarrow, bend left] (A);
		\node (cpt B) at (3.5,0) {$\setlength\arraycolsep{5pt}
			\begin{array}{c|cc}
				X  & Y\!{=}\T & Y\!{=}\F \\
			\hline
				\F & t_1      & 1-t_1    \\
				\T & t_2      & 1-t_2
			\end{array}$};
	\end{tikzpicture}
	\caption[]{\label{fig:cyclicBN}A cyclic GBN with CPTs for $X$ and $Y$}
\end{figure}
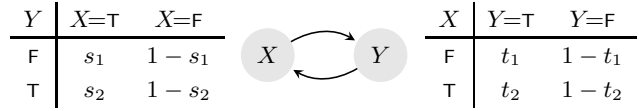

Imagine for instance an expert system for supporting diagnosis of
Covid-19, harvesting multiple clinical studies. One study might have
investigated the percentage of patients who have been diagnosed with
fever also having Covid-19, while another study in turn might have
investigated among the Covid-19 patients whether they have fever,
too. Clearly, both studies investigate the dependency between fever
and Covid-19, but under different conditions. Fever may weaken the immune system
and could increase the risk of a Covid-19 infection, while
Covid-19 itself has fever as a symptom.
In case there is uniform knowledge about ``which symptom was first''
in each of the constituent studies, then \emph{dynamic Bayesian
networks} (DBNs)~\cite{Mur02a} could be used as a model for the expert
system, breaking the interdependence of fever and Covid-19 through a
precedence relation. However, this implies either to rely only on studies
where these temporal dependencies are clearly identified or to introduce
an artificial notion of time that might lead to spurious results~\cite{Motzek2017}.
A naive encoding into the BN framework
always yields a graph structure
that contains cycles, as is the case in our small example
shown in Fig.\,\ref{fig:cyclicBN} where $X$ and $Y$ stand for the random variables
of diagnosing Covid-19 and fever, respectively. 

That cycles might be unavoidable has already been
observed in seminal papers such as~\cite{Pea90a,LauSpi88a}. %
But acyclicity is crucial for computing the joint probability distribution 
of a BN, and thereby is a prerequisite for, e.g., routine inference tasks.
Existing literature that considers cycles in BNs mainly recommends reducing
questions on the probability values to properties in acyclic BNs.
For instance, in \cite{JenAndKja87a} nodes are collapsed towards removing cycles,
while \cite{Pea90a} suggests to condition on each value combination on a cycle,
generating a decomposition into tree-like BNs and then averaging over the results
to replace cycles. Sometimes, application-specific methods that restructure
the cyclic BN towards an acyclic BN by introducing additional nodes~\cite{RobOwe97a,Han18a}
or by unrolling cycles up to a bounded depth~\cite{matthews2020cyclic,CasFerGar21a} 
have been reported to give satisfactory results. Other approaches either remove edges that have
less influence or reverse edges on cycles (see, e.g., \cite{JenNie07a}).
However, such approaches are highly application dependent and hinge on
knowledge about the context of the statistical data used to construct the BN.
Furthermore, as already pointed out by~\cite{TulNik05a}, they usually reduce 
the solution space of families of joint distributions to a single one,
or introduce solutions not consistent with the CPTs of the original cyclic BN.
While obviously many practitioners have stumbled on the problem how
to treat cycles in BNs and on the foundational question ``What is the meaning of a cyclic BN?'', 
there is very little work on the foundations of Bayesian reasoning with cycles.

In this paper, we approach this question by presenting general semantics for BNs
with cycles, together with algorithms to compute families of joint distributions for such BNs.
First, we investigate how the two main constituents of classical BNs, namely
consistency with the CPTs and independencies induced by the graph structure,
influence the joint distributions in the presence of cycles. 
This leads to \emph{constraints semantics} for cyclic BNs that comprise all those
joint distributions respecting the constraints, being either a single uniquely
defined one, none, or infinitely many distributions.
Second, we present semantics that formalize unfolding approaches and depend on the choice of a \emph{cutset}, a
set of random variables that break every cycle in a cyclic BN.
Intuitively, such cutsets form the seams along which feedback loops can be unraveled.
These semantics are defined in terms of the limit (or limit average) of a sequence of distributions at descending levels in the infinite unfolding of the BN.
We show that the same semantics can be defined using a Markov chain construction and subsequent long-run frequency analysis, which enables both precise computation of the semantics and
deep insights in the semantics' behavior.
Among others, an immediate result is that the family of distributions induced with respect to the limit semantics is always non-empty.
As we will argue, the limit semantics have 
obvious relations to a manifold of approaches that have appeared in the literature, 
yet they have not been spelled out and studied explicitly.

\subsection{Notation}

Let $\cV$ be a set of Boolean random variables%
\footnote{
	We use Boolean random variables for simplicity of representation,
	an extension of the proposed semantics over random variables with arbitrary finite state spaces is certainly possible.
}
over the domain $\Bool = \{\F, \T\}$.
We usually denote elements of $\cV$ by $X$, $Y\!$, or $Z$.
An \emph{assignment} over $\cV$ is a function $b\colon \cV \to \Bool$ which
we may specify through set notation, e.g., $b = \{X\!{=}\T, Y\!{=}\F\}$
for $b(X) = \T$ and $b(Y) = \F$,
or even more succinctly as $X\mnot{Y}$.
The set of all possible assignments over $\cV$ is denoted by $\Asg(\cV)$.
We write $b_{\cU}$ for the restriction of $b$ to a subset 
$\cU \subseteq \cV$, e.g., $b_{ \{X\} } = \{ X\!{=}\T \}$, and
may omit set braces, e.g.,\ $b_{X,Y} = b_{ \{X,Y\} }$.

A \emph{distribution} over a set $\cS$ is a function $\mu\colon \cS \to [0,1]$ where $\sum_{s \in \cS} \mu(s) = 1$. 
The set of all distributions over $\cS$ is denoted by $\Dist(\cS)$.
For $|\cS| = n$, $\mu$ will occasionally
be represented as a vector of size $n$ for some fixed order on $\cS$. %
In the following, we are mainly concerned with distributions over assignments, that is distributions 
$\mu \in \Dist(\Asg(\cV))$ for some set of random variables $\cV$.
Each such distribution~$\mu$ induces a probability measure (also called $\mu$) on $2^{\Asg(\cV)}$.
Thus, for a set of assignments $\phi \subseteq \Asg(\cV)$, we have
	$ \mu(\phi) = \sum_{b \in \phi} \mu(b)$. 
We are often interested in the probability of a \emph{partial assignment} $d \in \Asg(\cU)$ on a subset $\cU \subsetneq \cV$ of variables, 
which is given as the probability of the set of all full 
assignments $b \in \Asg(\cV)$ that agree with $d$ on $\cU$.
As a shorthand, we define
	\[ 
		\mu(d) 
		\ \coloneqq \
		\mu \bigl( \{ b \in \Asg(\cV) \ssep b_{\cU} = d \} \bigr)
		\ = \
		\sum_{\mathclap{\substack{b \in \Asg(\cV) \\ \text{s.t.\ } b_{ \cU} = d}}} \ 
		\mu(b).
	\]
The special case $\mu(X\!{=}\T)$ is called the \emph{marginal probability} of $X$.
The restriction of $\mu \in \Dist(\Asg(\cV))$ to $\cU$,  
denoted $\mu|_{\cU} \in \Dist(\Asg(\cU))$, is given by 
$\mu|_{\cU}(d) \coloneqq \mu(d)$.
For a set $\cW$ disjoint from $\cV$ and $\nu \in \Dist(\Asg(\cW))$, the \emph{product distribution} of $\mu$ and $\nu$ is given by 
$
	(\mu \combi \nu)(c) \coloneqq \mu(c_\cV) \cdot \nu(c_\cW)
$
for every $c \in \Asg(\cV \cup \cW)$.
$\mu $ is called a \emph{Dirac distribution} if 
$\mu(b) = 1$ for some assignment $b \in \Asg(\cV)$
and thus $\mu(c ) = 0$ for all other assignments $c \neq b$.
A Dirac distribution derived from a given assignment $b$ is denoted by $\Dirac(b)$.

\subsubsection{Graph Notations}
For a graph $\cG = \tuple{\cV, \cE}$ with nodes \cV and directed edges 
$\cE \subseteq \cV \times \cV$, we may represent an edge $(X,Y) \in \cE$
as $X \to Y$ if $\cE$ is clear from context.
$\Pre(X) \coloneqq \{Y\!\in \cV \ssep Y \to X\}$ denotes the set of 
\emph{parents} of a node $X\!\in \cV$, and 
$\PostS(X) \coloneqq \{Y\!\in \cV \ssep X \to \cdots \to Y\}$
is the set of nodes \emph{reachable} from $X$.
A node $X$ is called \emph{initial} if $\Pre(X) = \varnothing$,
and $\Init(\cG)$ is the set of all nodes initial in $\cG$.
A graph \cG is \emph{strongly connected} if each node in \cV is 
reachable from every other node.
A set of nodes \cD is a \emph{strongly connected component} (SCC) of \cG 
if all nodes in \cD can reach each other and \cD is not contained 
in another SCC, and a \emph{bottom SCC} (BSCC) if no node
in $\cV\setminus\cD$ can be reached from $\cD$.

\subsubsection{Markov Chains}\label{sec:MCs}

	A \emph{discrete-time Markov chain} (DTMC) is a tuple $\cM = \tuple{\cS, \bP}$ where \cS 
	is a finite set of states and $\bP\colon \cS \times \cS \to [0,1]$ a function such that 
	$\bP(s,\cdot) \in \Dist(\cS)$ for all states $s\in \cS$.
The underlying graph $\cG_\cM = \tuple{\cS, \cE}$ is defined by 
$\cE =\{ (s,t) \in \cS \times \cS \ssep \bP(s,t)>0\}$.
The transient distribution $\pi^\init_{n}\in \Dist(\cS)$ at step $n$ is defined through the 
probability $\pi^\init_{n}(s)$ to be in state $s$ after $n$ steps if starting with initial 
state distribution $\init$. It satisfies (in matrix-vector notation) $\pi^{\init}_n = \init \cdot \bP^n$. 
We are also interested in the long-run frequency of state occupancies when $n$ tends to infinity, 
defined as the Ces\`aro limit $\LongRunFreq^{\init}\colon \cS \to [0,1]$:
\begin{equation}\label{eq:lrf}
	\LongRunFreq^{\init}(s) 
	\ \coloneqq \, 
	\lim_{n \to \infty} \,
 	\frac{1}{n+1} \,
 	\sum^n_{i=0} \pi^{\init}_{n}(s) .
 	\tag{\scshape LRF}
\end{equation}
This limit always exists and corresponds to the long-run fraction of time spent in each state~\cite{KemSne1969}. 
The limit probability $\lim_{n \to \infty} \pi^{\init}_{n}$ is arguably more intuitive as a measure 
of the long-run behavior, but may not exist (due to periodicity). 
In case of existence, it agrees with the Ces\`aro limit $\LongRunFreq^{\init}$. 
If $\cG_\cM$ forms an SCC, the limit is independent of the choice of $\init$ and the 
superscript can be dropped. We denote this limit by $\LongRunFreq_\cM$.

\section{Generalized Bayesian Networks}

We introduce \emph{generalized Bayesian networks} (GBNs) as a BN model 
that does not impose acyclicity and comes with a distribution over initial nodes.

\begin{definition}[Generalized BN]\label{def:BN}
	A GBN $\cB$ is a tuple $\tuple{\cG, \cP, \init}$ where \vspace{-0.6em}
	\begin{itemize}
		\item $\cG = \tuple{\cV, \cE}$ is a directed graph with nodes \cV and an edge relation $\cE \subseteq \cV \times \cV$,
		\item $\cP$ is a function that maps all non-initial nodes 
			$X\!\in \cV{\setminus}\Init(\cG)$ paired with each of their parent assignments $b \in \Asg(\Pre(X))$ to a distribution \[\cP(X,b) \colon \Asg\bigl(\{X\}\bigl) \to [0,1],\] 
		\item $\init$ is a distribution over the assignments for the initial nodes 
			$\Init(\cG)$, i.e., $\init \in \Dist\bigl(\Asg(\Init(\cG))\bigl)$.
	\end{itemize}
\end{definition}%
The distributions $\cP(X,b)$ have the same role as the entries in a \emph{conditional probability table} 
(CPT) for $X$ in classical BNs:
they specify
the probability for $X\!{=}\T$ or $X\!{=}\F$ depending on the assignments of the predecessors of $X$.
To this end, for $X\!\in \cV{\setminus}\Init(\cG)$ and $b\in \Asg(\Pre(X))$, 
we also write  
$\CPTable(X\!{=}\T \mid b)$ for $\cP(X, b)(X\!{=}\T)$. 
In the literature, initial nodes are often assigned a marginal probability via a CPT as well,
assuming independence of all initial nodes. Differently, in our definition of GBNs,
it is possible to specify an arbitrary distribution $\init$ over all initial nodes.
If needed, \cP can be easily extended to initial nodes by setting 
$\cP(X,\varnothing) \coloneqq \init|_{\{X\}}$ for all $X\!\in \Init(\cG)$. 
Hence, classical BNs arise as a special instance of GBNs where the graph 
$\cG$ is acyclic and initial nodes are pairwise independent.
In that case, the CPTs given by $\cP$ are a compact representation 
of a single unique full joint distribution $\BNsem{\cB}$ over all random 
variables $X\!\in \cV$.   
For every assignment $b \in \Asg(\cV)$, we can compute $\BNsem{\cB}(b)$ by 
the so-called \emph{chain rule}:
\begin{equation}\label{BN}
	\BNsem{\cB}(b)
	\ \coloneqq\ 
	\init\bigl(b_{\Init(\cG)}\bigr) 
	\ \cdot \ 
	\prod_{\mathclap{ X \in \cV \setminus \Init(\cG) }} 
	\ \CPTable\bigl(b_X \mid b_{\Pre(X)}\bigr) .
	\tag{CR}
\end{equation}
In light of the semantics introduced later on, we define the \emph{standard BN-semantics} of an acyclic GBN \cB as the set $\semantics{\cB}_\BN \coloneqq \{\BNsem{\cB}\}$, and $\semantics{\cB}_\BN \coloneqq \varnothing$ if \cB contains cycles.

The distribution $\BNsem{\cB}$ satisfies two crucial properties:
First, it is consistent with the CPT entries given by \cP and the distribution $\init$,
and second, it observes the independencies encoded in the graph $\cG$.
In fact, those two properties are sufficient to uniquely characterize $\BNsem{\cB}$.
We briefly review the notion of independence and 
formally define CPT consistency later on in Section~\ref{sec:constraints}.

\subsubsection{Independence}

Any full joint probability distribution $\mu \in \Dist(\Asg(\cV))$ may induce
a number of conditional independencies among the random variables in $\cV$.
For $\cX$, $\cY$, and $\cZ$ disjoint subsets of $\cV$, the random variables in $\cX$ and $\cY$ are independent under $\mu$ given $\cZ$ if the conditional probability of each assignment over the nodes in $\cX$ 
given an assignment for $\cZ$ is unaffected by further conditioning on any assignment of $\cY$.
Formally, the set $\Indep(\mu)$ contains the triple $(\cX,\cY,\cZ)$
iff for all $a \in \Asg(\cX)$, $b \in \Asg(\cY)$, and $c \in \Asg(\cZ)$, we have
\[
    \mu(a \mid b, c)  =  \mu(a \mid c) 
    \quad \text{ or } \quad
    \mu(b,c) = 0.
\]
We also write $(\cX \perp \cY \mid \cZ)$ for $(\cX,\cY,\cZ) \in \Indep(\mu)$ and may omit the set brackets of $\cX$, $\cY$, and $\cZ$.

\subsubsection{d-separation}
For classical BNs, the graph topology encodes independencies that are necessarily satisfied by any full joint distribution regardless of the CPT entries. 
Given two random variables $X$ and $Y$ as well as a set of observed variables $\cZ$, then $X$ and $Y$ are conditionally independent given \cZ %
if the corresponding nodes in the graph are \emph{d-separated} given \cZ \cite{geiger1990d}.
To establish $d$-separation, all simple undirected paths\footnote{A path is simple if no node occurs twice in the path. ``Undirected'' in this context means that edges in either direction can occur along the path.} between $X$ and $Y$ need to be \emph{blocked} given $\cZ$.
Let $\mathbb{W}$ denote such a simple path 
$W_0, W_1, \dots, W_k$ with $W_0 = X$, $W_k = Y\!$, and either $W_i \to W_{i+1}$ or 
$W_i \leftarrow W_{i+1}$ for all $i<k$.
Then $\mathbb{W}$ is blocked given \cZ if and only if there exists an index $i$, $0<i<k$, such that one of the following two conditions holds:
(1) $W_i$ is in \cZ and is situated in a \emph{chain} or a \emph{fork} in $\mathbb{W}$, i.e.,
\begin{itemize}
	\item $W_{i-1} \to W_i \to W_{i+1}$ (forward chain)
	\item $W_{i-1} \leftarrow W_i \leftarrow W_{i+1}$ (backward chain) \quad and \quad $W_i \in \cZ$,
	\item $W_{i-1} \leftarrow W_i \to W_{i+1}$ (fork)
\end{itemize}
(2) $W_i$ is in a \emph{collider} and neither $W_i$ nor any descendant of $W_i$ is in $\cZ$, i.e.,
\begin{itemize}
	\item $W_{i-1} \to W_i \leftarrow W_{i+1}$ (collider) \quad and \quad 
	$\PostS(W_i) \cap \cZ = \varnothing$.
\end{itemize}
Two sets of nodes \cX and \cY are $d$-separated given a third set \cZ if for each $X\!\in \cX$ and $Y\!\in \cY$, $X$ and $Y$ are $d$-separated given $\cZ$.
Notably, the $d$-separation criterion is applicable also in presence of cycles \cite{spirtes1994conditional}.
For a graph $\cG = \tuple{\cV,\cE}$ of a GBN, we define the set $\dsep(\cG)$ as
\[ 
	\dsep(\cG) 
	\coloneqq 
	\bigl\{ 
		(\cX, \cY, \cZ) \in (2^\cV)^3 
		\ssep 
		\text{\cX and \cY are $d$-separated given \cZ}
	\bigr\}.
\]

For acyclic Bayesian networks it is well known that the independencies evident
from the standard BN semantics' distribution include the independencies derived from the graph.
That is, for acyclic GBNs $\cB_\nocycle = \tuple{\cG, \cP, \init}$ where all initial nodes are pairwise independent under $\init$, we have
\[
	\dsep(\cG) \ \subseteq \ \Indep\bigl( \BNsem{\cB_\nocycle} \bigr).
\]

For an arbitrary initial distribution, the above relation does not necessarily hold.
However, we can still find a set of independencies that are necessarily observed by the standard BN semantics and thus act as a similar lower bound.
We do so by assuming the worst case, namely that each initial node depends on every other initial node under $\init$.
Formally, given a graph $\cG = \tuple{\cV, \cE}$, we define a closure operation $\initCl(\cdot)$ as follows and compute the set $\dsep\bigl(\initCl(\cG)\bigr)$:
\[
	\initCl(\cG) 
	\ \coloneqq \ 
	\bigtuple{\cV, \cE \cup \{(A,B) \text{ for } A,B \in \Init(\cG), A \not= B\}}.
\]

\begin{lemma}\label{lemma:d-sep_init}
	Let $\cB_\nocycle = \tuple{\cG,\cP,\init}$ be an acyclic GBN.
	Then
	\[
		\dsep\bigl( \initCl(\cG) \bigr) 
		\ \subseteq \ 
		\Indep\bigl( \distBN(\cB_\nocycle) \bigr) .
	\]
\end{lemma}

As intuitively expected, the presence of cycles in \cG generally reduces the number of graph independencies, though note that also in strongly connected graphs independencies may exist.
For example, if \cG is a four-node cycle with nodes $W\!$, $X$, $Y\!$, and $Z$, then
$\dsep(\cG) = \bigl\{ (W \perp Y \mid X, Z), (X \perp Z \mid W, Y) \bigr\}$.

\section{Constraints Semantics}\label{sec:constraints}

For classical acyclic BNs there is exactly one distribution that agrees with all CPTs and satisfies the independencies encoded in the graph. 
This distribution can easily be constructed by means of the chain rule~\eqref{BN}.
For cyclic GBNs, applying the chain rule towards a full joint distribution is not 
possible in general, as the result is usually not a valid probability distribution.
Still, we can look for distributions consistent with a GBN's CPTs and the independencies derived from its graph.
Depending on the GBN, we will see that there may be none, exactly one, or even infinitely many distributions fulfilling these constraints.

\subsection{CPT-consistency}\label{sec:cpt}
We first provide a formal definition of CPT consistency in terms of linear constraints on full joint distributions.

\begin{definition}[Strong and weak CPT-consistency]\label{def:cpt-consistency}
	Let $\cB$ be a GBN with nodes $\cV$ and $X\!\in \cV$.
	Then $\mu$ is called \emph{strongly CPT-consistent for $X$ in \cB} (or simply \emph{CPT-consistent}) if for all $c \in \Asg(\Pre(X))$
	\begin{equation}\label{CPT}\tag{\CPT} 
		\mu(X\!{=}\T, c) 
		\ = \ 
		\mu(c) \cdot \CPTable( X\!{=}\T \mid c) .
	\end{equation}
	We say that $\mu$ is \emph{weakly CPT-consistent for $X$ in \cB} if
	\begin{equation}\label{WCPT}\tag{\wCPT} 
		\mu(X\!{=}\T) 
		\ = \ 
		\sum_{\mathclap{ c \in \Asg(\Pre(X)) }} \ \mu(c) \cdot \CPTable( X\!{=}\T \mid c ) .
	\end{equation}
\end{definition}

Intuitively, the constraint \eqref{CPT} is satisfied for $\dist$ if the
conditional probability $\dist( X\!{=}\T \mid c )$ equals the entry in the CPT for $X$ under assignment $c$, i.e., 
$\dist( X\!{=}\T \mid c ) = \CPTable(X\!{=}\T \mid c)$. %
In the weak case \eqref{WCPT}, only the resulting marginal probability of $X$
needs to agree with the CPTs.

\begin{definition}[\CPT and \wCPT semantics]\label{def:cpt-semantics}
	For a GBN $\cB = \tuple{\cG, \cP, \init}$, the \emph{CPT-semantics} $\semantics{\cB}_\CPT$ is 
	the set of all distributions $\mu \in \Dist(\Asg(\cV))$ where 
	$\mu|_{\Init(\cG)} = \init$ and
	$\mu$ is CPT-consistent for every node $X\!\in \cV{\setminus}\Init(\cG)$.
	The \emph{weak CPT-semantics} $\semantics{\cB}_\wCPT$ is defined analogously.
\end{definition}

Clearly, we have $\semantics{\cB}_\CPT \subseteq \semantics{\cB}_\wCPT$ for all $\cB$.
The next example shows that 
depending on the CPT values, the set $\semantics{\cB}_\CPT$ may be empty, a singleton, or of infinite cardinality.
\begin{example}\label{ex:cpt-semantics}
	To find CPT-consistent distributions for the GBN from Fig.\,\ref{fig:cyclicBN}, 
	we construct a system of linear equations whose solutions form distributions 
	$\mu \in \Dist\bigl(\Asg(\{X,Y\})\bigr)$, represented as vectors in the space $[0,1]^4$: 
	\[ \setlength\arraycolsep{5pt}
	\left( \begin{array}{cccc}
		s_1 & 0       & s_1{-}1 & 0        \\
		0   & s_2     & 0       & s_2{-}1  \\
		t_1 & t_1{-}1 & 0       & 0        \\
		0   & 0       & t_2     & t_2{-}1  \\
		1   & 1       & 1       & 1       
	\end{array}\right)
	\ \cdot \ 
	\left( \begin{array}{c}
		\mu_{\mnot{X} \mnot{Y}}\\
		\mu_{\mnot{X} Y}\\
		\mu_{X \mnot{Y}}\\
		\mu_{X Y}      
	\end{array}\right) 
	\ = \
	\left( \begin{array}{c}
		0\\0\\0\\0\\1      
	\end{array}\right)
	\]
	where%
	, e.g., $\mu_{X \mnot{Y}}$ abbreviates $\mu(X\!{=}\T, Y\!{=}\F)$. 
	The first line of the matrix states the \eqref{CPT} constraint for node $X$ and the parent assignment $c = \{Y\!{=}\F\}$:
	\begin{align*}
		0 \ &= \ s_1 \cdot \mu_{\mnot{X} \mnot{Y}} + 0 \cdot \mu_{\mnot{X} Y} + (s_1{-}1) \cdot \mu_{X \mnot{Y}} + 0 \cdot \mu_{X Y} \\
		\mu_{X \mnot{Y}} \ &= \ (\mu_{X \mnot{Y}} + \mu_{\mnot{X} \mnot{Y}}) \cdot s_1 \\
		\mu_{X \mnot{Y}} \ &= \ \mu_{\mnot{Y}} \cdot \CPTable(X\!{=}\T \mid Y\!{=}\F) \\
		\mu(X\!{=}\T, c) &= \mu(c) \cdot \CPTable(X\!{=}\T \mid c).
	\end{align*}
	Analogously, the following three rows encode the CPT constraints for $X$, $Y\!$, and their remaining parent assignments.
	The last row ensures that solutions are indeed probability distributions satisfying $\sum_c \mu(c) = 1$.

	The number of solutions for the system now depends on the CPT entries $s_1$, $s_2$, $t_1$, and $t_2$.
	For $s_1 = t_2 = 0$ and $s_2 = t_1 = 1$, no solution exists as the 
	first four equations require $\mu(b) = 0$ for all $b \in \Asg(\{X,Y\})$, while the 
	last equation ensures 
	$\mu_{\mnot{X} \mnot{Y}} + \mu_{\mnot{X} Y} + \mu_{X \mnot{Y}} + \mu_{X Y} = 1$.
	For $s_1 = t_1 = 0$ and $s_2 = t_2 = 1$, all distributions with 
	$\mu_{X Y} = 1 - \mu_{\mnot{X} \mnot{Y}}$ and 
	$\mu_{X \mnot{Y}} = \mu_{\mnot{X} Y} = 0$ are solutions.
	Finally, e.g., for $s_1 = t_1 = \nicefrac{3}{4}$ and $s_2 = t_2 = \nicefrac{1}{2}$, 
	there is exactly one solution with
	$\mu_{\mnot{X} \mnot{Y}} = \nicefrac{1}{10}$ 
	and $\mu(b) = \nicefrac{3}{10}$ for the other three assignments.
\end{example}

\subsection{Independence-consistency}

We extend \CPT semantics with a set of independencies that need to be observed by all induced distributions.
\begin{definition}[\normalfont{\CPTI} semantics]
	For a GBN $\cB = \tuple{\cG, \cP, \init}$ and a set of independencies $\cI$, the \emph{CPT-\cI semantics} $\semantics{\cB}_{\CPT\text{-}\cI}$ is defined as the set of all CPT-consistent distributions $\mu$ for which
	$\cI \subseteq \Indep(\mu)$ holds.
\end{definition}
Technically, the distributions in $\semantics{\cB}_\CPTI$
have to fulfill the following polynomial constraints in addition to the CPT-consistency constraints:
\begin{equation}\tag{\CPTI}
	\dist( b ) \cdot \dist( b_\cW ) 
	\ = \
	\dist( b_{\{X\} \cup \cW} ) \cdot \dist( b_{\cU \cup \cW} )
  	\label{CPTI}
\end{equation}
for each independence $(X\! \perp \cU \mid \cW) \in \cI$ with variable $X{\in} \cV$ 
and sets of variables $\cU, \cW \subseteq \cV$, and for each assignment $b \in \Asg(\{X\} \cup \cU \cup \cW)$.
Note that in case $\dist(b_{\cW}) >0$, \eqref{CPTI} is equivalent to the constraint
$\dist( b_X \mid b_{\cU \cup \cW} ) = \dist( b_X \mid b_{\cW} )$.

We can now formally state the alternative characterization of the standard BN semantics as the unique CPT-consistent distribution that satisfies the $d$-separation independencies of the graph.
For each classical BN $\cB$ with acyclic graph $\cG$ and $\cI = \dsep(\cG)$, we have $\semantics{\cB}_\BN = \{\distBN(\cB)\} = \semantics{\cB}_\CPTI$.
Thus, the \CPTI semantics provides a conservative extension of the standard BN semantics to GBNs with cycles.
However, in practice, its use is limited since there might be no distribution 
that satisfies all constraints.
In fact, the case where $\semantics{\cB}_\CPTI = \varnothing$ is to be expected for most cyclic GBNs, given that the resulting constraint systems tend to be heavily over-determined.

The next section introduces semantics that follow a more constructive approach.
We will see later on in Section~\ref{sec:limit_semantics_properties} that the families of distributions induced by these semantics are always non-empty and usually singletons.

\section{Limit and Limit Average Semantics}

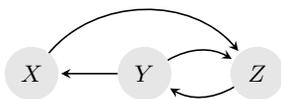
\begin{figure}[t]
	\centering
	\begin{tikzpicture}
		\node (X) at (-1.5, 0) [state] {$X$};
		\node (Y) at ( 0.0, 0) [state] {$Y$};
		\node (Z) at ( 1.5, 0) [state] {$Z$};
		
		\draw[myarrow, bend left=50] (X) edge (Z);
		\draw[myarrow, bend left] (Y) edge (Z);
		\draw[myarrow, bend left] (Z) edge (Y);
		\draw[myarrow] (Y) edge (X);
		
	\end{tikzpicture}
	\caption[]{\label{fig:scGBN} The graph of a strongly connected GBN}
\end{figure}

\begin{figure}[t]
\centering
\begin{subfigure}{.5\textwidth}
	\centering
	\begin{tikzpicture}
		\node (X0) at (-1.5, 0) [state] {$X_0$};
		\node (Y0) at ( 0.0, 0) [state] {$Y_0$};
		\node (Z0) at ( 1.5, 0) [state] {$Z_0$};

		\node (X1) at (-1.5, -1.5) [state] {$X_1$};
		\node (Y1) at ( 0.0, -1.5) [state] {$Y_1$};
		\node (Z1) at ( 1.5, -1.5) [state] {$Z_1$};
				
		\draw[myarrow] (X0) -- (Z1);
		\draw[myarrow] (Y0) -- (Z1);
		\draw[myarrow] (Z0) -- (Y1);
		\draw[myarrow] (Y0) -- (X1);
		
		\node (X2) at (-1.5, -3) [state] {$X_2$};
		\node (Y2) at ( 0.0, -3) [state] {$Y_2$};
		\node (Z2) at ( 1.5, -3) [state] {$Z_2$};
		
		\draw[myarrow] (X1) -- (Z2);
		\draw[myarrow] (Y1) -- (Z2);
		\draw[myarrow] (Z1) -- (Y2);
		\draw[myarrow] (Y1) -- (X2);
		
		\node (X3) at (-1.5, -4.5) [state, opacity=0.2] {};
		\node (Y3) at ( 0.0, -4.5) [state, opacity=0.2] {};
		\node (Z3) at ( 1.5, -4.5) [state, opacity=0.2] {};
		\node (dots) at (0.0, -4.05) {\vdots};
		
		\draw[fade out arrow] (X2) -- (Z3);
		\draw[fade out arrow] (Y2) -- (Z3);
		\draw[fade out arrow] (Z2) -- (Y3);
		\draw[fade out arrow] (Y2) -- (X3);		
	\end{tikzpicture}
	\caption[]{\label{fig:unfolding1}Unfolding along all nodes}
\end{subfigure}%
\begin{subfigure}{.5\textwidth}
	\centering
	\begin{tikzpicture}
		\node (Z0) at ( 1.5, 0) [state] {$Z_0$};

		\node (X1) at (-1.5, -1.5) [state] {$X_1$};
		\node (Y1) at ( 0.0, -1.5) [state] {$Y_1$};
		\node (Z1) at ( 1.5, -1.5) [state] {$Z_1$};
				
		\draw[myarrow] (Z0) -- (Y1);
		
		\node (X2) at (-1.5, -3) [state] {$X_2$};
		\node (Y2) at ( 0.0, -3) [state] {$Y_2$};
		\node (Z2) at ( 1.5, -3) [state] {$Z_2$};
		
		\draw[myarrow, bend left=50] (X1) edge (Z1);
		\draw[myarrow] (Y1) -- (Z1);
		\draw[myarrow] (Z1) -- (Y2);
		\draw[myarrow] (Y1) -- (X1);
		
		\node (X3) at (-1.5, -4.5) [state, opacity=0.2] {};
		\node (Y3) at ( 0.0, -4.5) [state, opacity=0.2] {};
		\node (Z3) at ( 1.5, -4.5) [state, opacity=0.2] {};
		\node (dots) at (0.0, -4.05) {\vdots};
		
		\draw[myarrow, bend left=50] (X2) edge (Z2);
		\draw[myarrow] (Y2) -- (Z2);
		\draw[fade out arrow] (Z2) -- (Y3);
		\draw[myarrow] (Y2) -- (X2);
		
		\draw[myarrow, bend left=50, opacity=0.1] (X3) edge (Z3);
		\draw[myarrow, opacity=0.1] (Y3) -- (Z3);
		\draw[myarrow, opacity=0.1] (Y3) -- (X3);
	\end{tikzpicture}
	\caption[]{\label{fig:unfolding2}Unfolding along the $Z$ nodes}
\end{subfigure}
\caption{Two infinite unfoldings of the graph in Fig.\,\ref{fig:scGBN}}
\label{fig:unfold}
\end{figure}
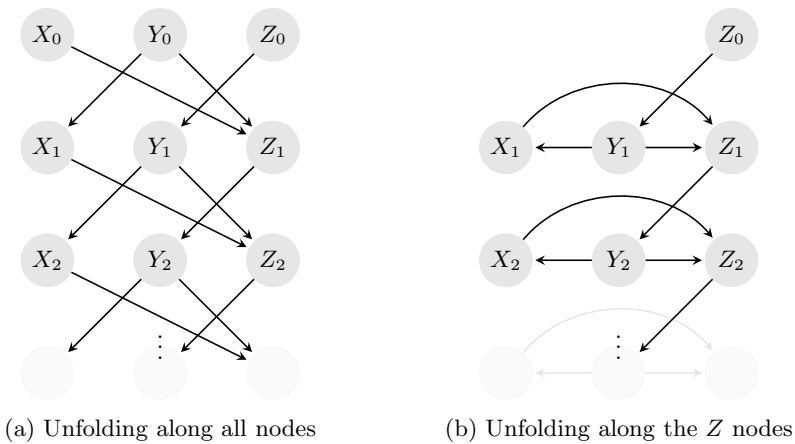

We first develop the basic ideas underling the semantics by following an example, before giving a formal treatment in Section~\ref{sec:limit_semantics_formal}.

\subsection{Intuition}

Consider the GBN \cB whose graph \cG is depicted in Fig.\,\ref{fig:scGBN}.
One way to get rid of the cycles is to construct an infinite unfolding of \cB as shown in Fig.\,\ref{fig:unfolding1}.
In this new graph $\cG_\infty$, each level contains a full copy of the original nodes and corresponds to some $n \in \Nat$.
For any edge $X \to Y$ in the original graph, we add edges $X_n \to Y_{n+1}$ to $\cG_\infty$, such that each edge descends one level deeper.
Clearly any graph constructed in this way is acyclic, but this fact alone does not aid in finding a matching distribution since we dearly bought it by giving up finiteness.
However, we can consider what happens when we plug in some initial distribution~$\mu_0$ over the nodes $X_0$, $Y_0$, and $Z_0$.
Looking only at the first two levels, we then get a fully specified acyclic BN by using the CPTs given by \cP for $X_1$, $Y_1$, and $Z_1$.
For this sub-BN, the standard BN semantics yields a full joint distribution over the six nodes from $X_0$ to $Z_1$, which also induces a distribution $\mu_1$ over the three nodes at level $1$.
This procedure can then be repeated to construct a distribution $\mu_2$ over the nodes $X_2$, $Y_2$, and $Z_2$, and, more generally, to get a distribution $\mu_{n+1}$ given a distribution $\mu_n$.
Recall that each of those distributions can be viewed as vector of size $2^3$.
Considering the sequence $\mu_0, \mu_1, \mu_2, \dots$, the question naturally arises whether 
a limit exists, i.e., a distribution/vector $\mu$ such that \vspace{-0.5em}
\[ 
	\mu \ \ = \ \ \lim_{\mathclap{ n \to \infty }} \ \mu_n.
\]

\begin{example}\label{ex:no_limit}
	Consider the GBN from Fig.\,\ref{fig:cyclicBN} with CPT entries $s_1 = t_2 = 1$ and $s_2 = t_1 = 0$, which intuitively describe the contradictory dependencies ``$X$ iff not $Y$'' and ``$Y$ iff $X$''.
	For any initial distribution $\mu_0 = \tuple{e\ f\ g\ h}$, the construction informally described above yields the following sequence of distributions $\mu_n$:
	\[
		\mu_0 = \left(\begin{array}{c}e\\ f\\ g\\ h\end{array}\right)\!, \
		\mu_1 = \left(\begin{array}{c}f\\ h\\ e\\ g\end{array}\right)\!, \
		\mu_2 = \left(\begin{array}{c}h\\ g\\ f\\ e\end{array}\right)\!, \
		\mu_3 = \left(\begin{array}{c}g\\ e\\ h\\ f\end{array}\right)\!, \
		\mu_4 = \left(\begin{array}{c}e\\ f\\ g\\ h\end{array}\right)\!, \ \dots
	\]
As $\mu_4 = \mu_0$, the sequence starts to cycle infinitely between the first four distributions. The series converges for $e=f=g=h=\nicefrac{1}{4}$ (in which case the sequence is constant), but does not converge for any other initial distribution.
\end{example}

The example shows that the existence of the limit depends on the given initial distribution.  
In case no limit exists because some distributions keep repeating without ever converging, it is possible to determine the \emph{limit average} (or \emph{Cesàro limit}) of the sequence: \vspace{-0.8em}
\[
	\tilde{\mu}
	\ \ = \ \ 
	\lim_{\mathclap{ n \to \infty }} \ \ \frac{1}{n+1}\ \sum_{i=0}^{n} \mu_i.
\]
The limit average has three nice properties: 
First, if the regular limit $\mu$ exists, then the limit average $\tilde{\mu}$ exists as well and is identical to $\mu$.
Second, in our use case, $\tilde{\mu}$ in fact always exists for any initial distribution $\mu_0$.
And third, as we will see in Section~\ref{sec:cutsetMC}, the limit average corresponds to the long-run frequency of certain Markov chains, which allows us both to explicitly compute and to derive important properties of the limit distributions.
\begin{example}
	Continuing Ex.\ \ref{ex:no_limit}, the limit average of the sequence $\mu_0, \mu_1, \mu_2, \dots$ is the uniform distribution $\tilde{\mu} = \tuple{\nicefrac{1}{4}\ \nicefrac{1}{4}\ \nicefrac{1}{4}\ \nicefrac{1}{4} }$, regardless of the choice of~$\mu_0$. %
\end{example}
 
Before we formally define the infinite unfolding of GBNs and the resulting limit semantics, there is one more observation to be made.
To ensure that the unfolded graph $\cG_\infty$ is acyclic, we redirected every edge of the GBN \cB to point one level deeper, resulting in the graph displayed in Fig.\,\ref{fig:unfolding1}.
As can be seen in Fig.\,\ref{fig:unfolding2}, we also get an acyclic unfolded graph by only redirecting the edges originating in the $Z$ nodes to the next level and keeping all other edges on the same level.
The relevant property is to pick a set of nodes such that for each cycle in the original GBN $\cB$, at least one node in the cycle is contained in the set.
We call such sets the \emph{cutsets} of $\cB$.

\begin{definition}[Cutset]
	Let \cB be an GBN with graph $\cG = \tuple{\cV, \cE}$.
	A subset $\cC \subseteq \cV$ is a \emph{cutset} for \cB if every cycle in \cG contains at least one node from $\cC$.
\end{definition}
\begin{example}
	The GBN in Fig.\,\ref{fig:scGBN} has the following cutsets: $\{Y\}$, $\{Z\}$, $\{X, Y\}$, $\{X, Z\}$, $\{Y, Z\}$, and $\{X, Y, Z\}$.
	Note that $\{X\}$ does not form a cutset as no node from the cycle $Y\!\to Z \to Y$ is contained.
\end{example}

So far we implicitly used the set \cV of all nodes for the unfolding, which always trivially forms a cutset.
The following definitions will be parameterized with a cutset, 
as the choice of cutsets influences the resulting distributions as well as the time complexity.

\subsection{Formal Definition}\label{sec:limit_semantics_formal}

Let $\cV_n \coloneqq \{X_n \ssep X\! \in \cV\}$ denote the set of nodes on the $n^\text{th}$ level of the unfolding in $\cG_\infty$.
For $\cC \subseteq \cV$ a cutset of the GBN, the subset of cutset nodes on that level is given by $\cC_n \coloneqq \{X_n \!\in \cV_n \ssep X\!\in\cC\}$.
Then a distribution $\gamma_n \in \Dist(\Asg(\cC_n))$ for the cutset nodes in $\cC_n$ suffices to get a full distribution $\mu_{n+1} \in \Dist(\Asg(\cV_{n+1}))$ over all nodes on the next level, $n+1$:
We look at the graph fragment $\cG_{n+1}$ of $\cG_\infty$ given by the nodes $\cC_n \cup \cV_{n+1}$ and their respective edges. 
In this fragment, the cutset nodes are initial, so the cutset distribution $\gamma_n$ can be combined with the initial distribution~$\init$ to act as new initial distribution.
For the nodes in $\cV_{n+1}$, the corresponding CPTs as given by $\cP$ can be used, i.e., $\cP_n(X_n, \cdot) = \cP(X, \cdot)$ for $X_n \in \cV_n$.
Putting everything together, we obtain an acyclic GBN $\cB_{n+1} = \tuple{\cG_{n+1}, \cP_{n+1}, \init \combi \gamma_n}$.
However, GBNs constructed in this way for each level $n>0$ are all isomorphic and only differ in the given cutset distribution~$\gamma$.
For simplicity and in light of later use, we thus define a single representative GBN $\dissectedBN{\cB}{\cC}{\gamma}$ that represents a dissection of $\cB$ along a given cutset $\cC$, with $\init \combi \gamma$ as initial distribution.
\begin{definition}[Dissected GBN]\label{def:dissected_GBN}
	Let $\cB = \tuple{\cG, \cP, \init}$ be a GBN with graph $\cG = \tuple{\cV, \cE}$ and $\cC \subseteq \cV$ a cutset for \cB with distribution 
	$\gamma \in \Dist(\Asg(\cC))$. 
	Then, the \emph{\cC-dissected GBN} 
	$\dissectedBN{\cB}{\cC}{\gamma}$ %
	is the acyclic GBN $\tuple{\cG_\cC, \cP_\cC, \init \combi \gamma}$ with
	graph $\cG_\cC = \tuple{\cV \cup \cC', \cE_\cC}$ where
	\begin{itemize}\setlength{\itemsep}{-0.01em}
	\item $\cC' \coloneqq \{X' \ssep X\!\in \cC\}$ extends $\cV$ by fresh copies of all cutset nodes; 
	\item 
		incoming edges to nodes in \cC are redirected to their copies, i.e., 
		\[
			\cE_{\cC} \coloneqq \bigl\{ (X,Y') \ssep (X,Y) \in \cE, Y\!\in \cC \bigr\}
			               \cup \bigl\{ (X, Y) \ssep (X,Y) \in \cE, Y\!\notin \cC \bigr\};
		\]
	\item 
		the function $\cP_\cC$ uses the CPT entries given by \cP for the cutset nodes 
		as entries for their copies and the original entries for all other nodes, i.e., 
		we have $\cP_\cC(Y'\!,a) = \cP(Y,a)$ for each node $Y'\!\in \cC'$ and 
		parent assignment $a \in \Asg(\Pre(Y'))$, and 
		$\cP_\cC(X,b) = \cP(X,b)$ for $X\!\in \cV{\setminus}\cC$ and $b \in \Asg(\Pre(X))$.
	\end{itemize}
\end{definition}

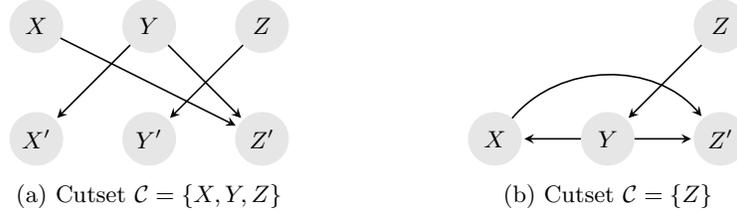
\begin{figure}[t]
\centering
\begin{subfigure}{.5\textwidth}
	\centering
	\begin{tikzpicture}
		\node (X0) at (-1.5, 0) [state] {$X$};
		\node (Y0) at ( 0.0, 0) [state] {$Y$};
		\node (Z0) at ( 1.5, 0) [state] {$Z$};

		\node (X1) at (-1.5, -1.5) [state] {$X'$};
		\node (Y1) at ( 0.0, -1.5) [state] {$Y'$};
		\node (Z1) at ( 1.5, -1.5) [state] {$Z'$};
				
		\draw[myarrow] (X0) -- (Z1);
		\draw[myarrow] (Y0) -- (Z1);
		\draw[myarrow] (Z0) -- (Y1);
		\draw[myarrow] (Y0) -- (X1);	
	\end{tikzpicture}
	\caption[]{\label{fig:dissectedBN1}Cutset $\cC = \{X,Y,Z\}$}
\end{subfigure}%
\begin{subfigure}{.5\textwidth}
	\centering
	\begin{tikzpicture}
		\node (Z0) at ( 1.5, 0) [state] {$Z$};

		\node (X1) at (-1.5, -1.5) [state] {$X$};
		\node (Y1) at ( 0.0, -1.5) [state] {$Y$};
		\node (Z1) at ( 1.5, -1.5) [state] {$Z'$};
				
		\draw[myarrow] (Z0) -- (Y1);
		\draw[myarrow, bend left=50] (X1) edge (Z1);
		\draw[myarrow] (Y1) -- (Z1);
		\draw[myarrow] (Y1) -- (X1);	
	\end{tikzpicture}
	\caption[]{\label{fig:dissectedBN2}Cutset $\cC = \{Z\}$}
\end{subfigure}
\caption{Dissections of the GBN in Fig.\,\ref{fig:scGBN} for two cutsets}
\label{fig:dissect}
\end{figure}
\noindent
Fig.~\ref{fig:dissect} shows two examples of dissections on the GBN of Fig.~\ref{fig:scGBN}.
As any dissected GBN is acyclic by construction, the standard BN semantics yields a full joint distribution over all nodes in $\cV \cup \cC'$.
We restrict this distribution to the nodes in $(\cV\setminus\cC) \cup \cC'$, as those are the ones on the ``next level'' of the unfolding, while re-identifying the cutset node copies with the original nodes to get a distribution over $\cV$.
Formally, we define the distribution $\Next(\cB, \cC, \gamma)$ for each assignment $b \in \Asg(\cV)$ as
\[
	\Next(\cB, \cC, \gamma)(b) 
	\ \coloneqq \ 
	\distBN\bigl( \dissectedBN{\cB}{\cC}{\gamma} \bigr)(b')
\]
where the assignment $b' \in \Asg\bigl((\cV{\setminus}\cC) \cup \cC'\bigr)$ is given by $b'(X) = b(X)$ 
for all $X\!\in \cV{\setminus}\cC$ and $b'(Y') = b(Y)$ for all $Y\!\in \cC$.
In the unfolded GBN, this allows us to get from a cutset distribution $\gamma_n$ to the next level distribution $\mu_{n+1} = \Next(\cB, \cC, \gamma_n)$.
The next cutset distribution $\gamma_{n+1}$ is then given by restricting the full distribution to the nodes in $\cC$, i.e., $\gamma_{n+1} = \Next(\cB, \cC, \gamma_n)|_\cC$.%
\footnote{Recall that we may view distributions as vectors which allows us to equate distributions over different but isomorphic domains.}
Vice versa, a cutset distribution $\gamma$ suffices to recover the full joint distribution over all nodes $\cV$.
Again using the standard BN semantics of the dissected GBN, we define the distribution $\Extend(\cB,\cC,\gamma) \in \Dist(\Asg(\cV))$ as
\[
	\Extend(\cB,\cC,\gamma) 
	\ \coloneqq \ 
	\distBN\bigl( \dissectedBN{\cB}{\cC}{\gamma} \bigr)\big|_\cV .
\]

With these definitions at hand, we can formally define the limit and limit average semantics described in the previous section.

\begin{definition}[Limit and limit average semantics]\label{def:limit_semantics}
	Let \cB be a GBN over nodes $\cV$ with cutset $\cC$.
	The \emph{limit semantics of \cB w.r.t.\ \cC} is the partial function 
	\[
		\mathit{Lim}(\cB,\cC, \cdot) 
		\ \colon \
		\Dist\bigl(\Asg(\cC)\bigr) 
		\rightharpoonup 
		\Dist\bigl(\Asg(\cV)\bigr)
	\]
	from initial cutset distributions $\gamma_0$ to full distributions $\mu = \Extend(\cB, \cC, \gamma)$ where
	\[
		\gamma
		\ = 
		\lim_{n \to \infty} 
		\gamma_n 
		\qquad
		\text{ and }
		\qquad
		\gamma_{n+1} = \Next(\cB, \cC, \gamma_n)|_\cC .
	\]
	The set $\semantics{\cB}_\LimC$ is given by the image of $\mathit{Lim}(\cB,\cC,\cdot)$, i.e.,
	\[
		\semantics{\cB}_\LimC \coloneqq \{\mathit{Lim}(\cB,\cC,\gamma_0) \ssep \gamma_0 \in \Dist(\Asg(\cC))\text{ s.t.\ $\mathit{Lim}(\cB,\cC,\gamma_0)$ is defined} \}.
	\]
	The \emph{limit average semantics of \cB w.r.t.\ \cC} is the partial function 
	\[
		\mathit{LimAvg}(\cB,\cC, \cdot)
		\ \colon \
		\Dist\bigl(\Asg(\cC)\bigr) 
		\rightharpoonup 
		\Dist\bigl(\Asg(\cV)\bigr)
	\]
	from $\gamma_0$ to distributions  $\mu = \Extend(\cB, \cC, \gamma)$ where
	\[
		\gamma
		\ = 
		\lim_{n \to \infty}
		\frac{1}{n+1} 
		\sum_{i=0}^{n}
		\gamma_n 
		\qquad
		\text{ and }
		\qquad
		\gamma_{n+1} = \Next(\cB, \cC, \gamma_n)|_\cC .
	\]
	The set $\semantics{\cB}_\LimAvgC$ is likewise given by the image of $\mathit{LimAvg}(\cB,\cC,\cdot)$.
\end{definition}

We know that the limit average coincides with the regular limit if the latter exists, so for every 
initial cutset distribution $\gamma_0$, we have $\mathit{Lim}(\cB,\cC,\gamma_0) = \mathit{LimAvg}(\cB,\cC,\gamma_0)$ if $\mathit{Lim}(\cB,\cC,\gamma_0)$ is defined.
Thus, $\semantics{\cB}_\LimC \subseteq \semantics{\cB}_\LimAvgC$.

\section{Markov Chain Semantics}\label{sec:cutsetMC}

While we gave some motivation for the limit and limit average semantics, their definitions do not reveal an explicit way to compute their member distributions.
In this section we introduce the \emph{(cutset) Markov chain semantics} which offers explicit construction of distributions and is shown to coincide with the limit average semantics.
It further paves the way for proving several properties of both limit semantics in Section~\ref{sec:limit_semantics_properties}.

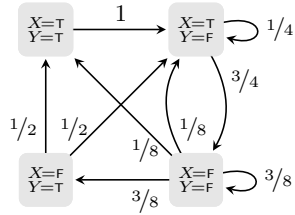
\begin{figure}[t]
	\centering
	\begin{tikzpicture}
		\node (AB)   at (0.0, 0.0) [MCstate] {$\substack{ X\!{=}\T \\ Y\!{=}\T }$}; %
		\node (AnB)  at (2.0, 0.0) [MCstate] {$\substack{ X\!{=}\T \\ Y\!{=}\F }$}; %
		\node (nAB)  at (0.0,-2.0) [MCstate] {$\substack{ X\!{=}\F \\ Y\!{=}\T }$}; %
		\node (nAnB) at (2.0,-2.0) [MCstate] {$\substack{ X\!{=}\F \\ Y\!{=}\F }$}; %
		\draw[long myarrow, loop right] (nAnB) edge node[label]{$\nicefrac{3}{8}$} (nAnB);
		\draw[long myarrow]             (nAnB) edge node[label, below, near start]{$\nicefrac{3}{8}$} (nAB);
		\draw[long myarrow, bend left]  (nAnB) edge node[label, right, near start]{$\nicefrac{1}{8}$} (AnB);
		\draw[long myarrow]             (nAnB) edge node[label, below, near start]{$\nicefrac{1}{8}$} (AB);

		\draw[long myarrow] (nAB) edge node[label, left, near start]{$\nicefrac{1}{2}$} (AB);
		\draw[long myarrow] (nAB) edge node[label, left, near start]{$\nicefrac{1}{2}$} (AnB);

		\draw[long myarrow, bend left]  (AnB) edge node[label, right, near start]{$\nicefrac{3}{4}$} (nAnB);
		\draw[long myarrow, loop right] (AnB) edge node[label]{$\nicefrac{1}{4}$} (AnB);
		
		\draw[long myarrow] (AB) edge node[label, above]{$1$} (AnB);
		\node (AB)   at (0.0, 0.0) [MCstate] {$\substack{ X\!{=}\T \\ Y\!{=}\T }$}; %
		\node (AnB)  at (2.0, 0.0) [MCstate] {$\substack{ X\!{=}\T \\ Y\!{=}\F }$}; %
		\node (nAB)  at (0.0,-2.0) [MCstate] {$\substack{ X\!{=}\F \\ Y\!{=}\T }$}; %
		\node (nAnB) at (2.0,-2.0) [MCstate] {$\substack{ X\!{=}\F \\ Y\!{=}\F }$}; %
	\end{tikzpicture}
	\caption{\label{fig:cutset_MC}A cutset Markov chain for a cutset $\cC = \{X,Y\}$}
\end{figure}

At the core of the cutset Markov chain semantics lies the eponymous \emph{cutset Markov chain} which captures how probability mass flows from one cutset assignment to the others.
To this end, the Dirac distributions corresponding to each assignment are used as initial distributions in the dissected GBN.
With the \Next\ function we then get a new distribution over all cutset assignments, and the probabilities assigned by this distribution are used as transition probabilities for the Markov chain.

\begin{definition}[Cutset Markov chain]\label{def:cutsetMC}
	Let \cB be a GBN with cutset \cC.
	The \emph{cutset Markov chain} $\CutsetMC{\cB}{\cC} = \tuple{\Asg(\cC), \bP}$ w.r.t.\ \cB and \cC 
	is a DTMC where the transition matrix $\bP\!$ is given for all cutset assignments $b,c \in \Asg(\cC)$ by
	\[ 
		\bP(b,c)
		\ \coloneqq \
		\Next\bigl(\cB, \cC, \Dirac(b)\bigr)(c) .
	\]
\end{definition}

\begin{example}\label{ex:cutset_MC}
	Fig.\,\ref{fig:cutset_MC} shows the cutset Markov chain for the GBN from Fig.\,\ref{fig:cyclicBN} 
	with CPT entries $s_1 = \nicefrac{1}{4}$, $s_2 = 1$, $t_1 = \nicefrac{1}{2}$, $t_2 = 0$, and cutset $\cC = \{X,Y\}$.
	Exemplarily, the edge at the bottom from assignment $b = \{X\!{=}\F, Y\!{=}\F\}$ to assignment $c = \{X\!{=}\F, Y\!{=}\T\}$ 
	with label $\nicefrac{3}{8}$ is derived as follows:
	\begin{eqnarray*}
		\bP(b,c) &\ \ =\ \ & \Next\bigl(\cB, \cC, \Dirac(b)\bigr)(c)
		 \ = \ \distBN\bigl( \dissectedBN{\cB}{\cC}{\Dirac(b)} \bigr)(c') \\
		 &=& \sum\nolimits_{\substack{a \in \Asg(\cV_\cC) \\ \text{\ s.t.\ } c' \subseteq a}} \ \distBN\bigl( \dissectedBN{\cB}{\cC}{\Dirac(b)} \bigr)(a) \\
		 &=& \sum\nolimits_{\substack{a \in \Asg(\cV_\cC) \\ \text{\ s.t.\ } c' \subseteq a}} \ \Dirac(b)(a_{X,Y}) \cdot  \CPTable(X'\!{=}\F \mid a_Y) \cdot \CPTable(Y'\!{=}\T \mid a_X) \\
		 &=& \CPTable(X'\!{=}\F \mid Y\!{=}\F) \cdot \CPTable(Y'\!{=}\T \mid X\!{=}\F)
		 \ = \ (1 - s_1) \cdot t_1 \ \ =\ \  \nicefrac{3}{8}.
	\end{eqnarray*}
	Note that in the second-to-last step, in the sum over all full assignments $a$ which agree with the partial assignment $c'$, only the assignment which also agrees with $b$ remains as for all other assignments we have $\Dirac(b)(a_{X,Y}) = 0$.
\end{example}

Given a cutset Markov chain with transition matrix $\bP$ 
and an initial cutset distribution $\gamma_0$, we can compute the uniquely defined long-run frequency distribution $\LongRunFreq^{\gamma_0}$ (see Section~\ref{sec:MCs}).
Then the Markov chain semantics is given by the extension of this distribution over the whole GBN.

\begin{definition}[Markov chain semantics]\label{def:cutset_semantics}
	Let \cB be a GBN over nodes~\cV with a cutset $\cC \subseteq \cV$ 
	and cutset Markov chain $\CutsetMC{\cB}{\cC} = \tuple{\Asg(\cC), \bP}$.
	Then the \emph{Markov chain semantics of \cB w.r.t.\ \cC} is the function 
	\[
		\mathit{MCS}(\cB,\cC, \cdot) 
		\ \colon \ 
		\Dist\bigl(\Asg(\cC)\bigr) \to \Dist\bigl(\Asg(\cV)\bigr)
	\]
	from cutset distributions $\gamma_0$ to full distributions 
	$\mu = \Extend(\cB,\cC,\LongRunFreq^{\gamma_0})$ where
   	\[
    	\LongRunFreq^{\gamma_0}
    	\ = \ 
    	\lim\limits_{n\to \infty}\  
    	\frac{1}{n{+}1}\ 
    	\sum_{i=0}^n \gamma_i
    	\qquad
    	\text{ and }
    	\qquad
    	\gamma_{i+1} =\gamma_i \cdot \bP.
   	\]
	The set $\semantics{\cB}_{ \Cut\text{-}\cC }$ is defined as
	the image of $\mathit{MCS}(\cB,\cC, \cdot)$.
\end{definition}

In the following lemma, we give four equivalent characterizations 
of the long-run frequency distributions of the cutset Markov chain.
\begin{lemma}\label{lemma:convex_combi_BSCCs}
	Let $\cB$ be a GBN with cutset $\cC$, cutset distribution $\gamma \in \Dist(\Asg(\cC))$, and $\cM = \tuple{\Asg(\cC), \bP}$ the cutset Markov chain $\CutsetMC{\cB}{\cC}$.
	Then the following statements are equivalent:
 	\begin{enumerate}[label=(\alph*),leftmargin=*] %
 	\item 
 		$\gamma =\gamma \cdot \bP$. %
	\item
		There exists $\gamma_0\in \Dist(\Asg(\cC))$ such that for $\gamma_{i+1} =\gamma_i \cdot \bP$, we have
		\[
			\gamma\ = \ \lim\limits_{n\to \infty}\  
			\frac{1}{n{+}1}\ \sum_{i=0}^n \gamma_i .
		\]
	\item %
		$\gamma$ belongs to the convex hull of the long-run
		frequency distributions $\LongRunFreq_\cD$ of the bottom SCCs $\cD$ of $\cM$.
 	\item 
 		$\gamma = \Next(\cB, \cC, \gamma)|_\cC$.
  \end{enumerate}
\end{lemma}

Following Lemma~\ref{lemma:convex_combi_BSCCs}, we can equivalently define the cutset Markov chain semantics as the set of extensions of all stationary distributions for $\bP$:
\[ 
	\semantics{\cB}_\CutC 
	\ \coloneqq \
	\big\{ 
		\Extend(\cB,\cC,\gamma)
		\ssep 
		\gamma \in \Dist\bigl(\Asg(\cC)\bigl)
		\text{ s.t. }
		\gamma = \gamma \cdot \bP
	\big\} . 
\]

\begin{example}\label{ex:cutset_semantics}
	Continuing Ex.~\ref{ex:cutset_MC}, there is a unique stationary distribution $\gamma$ with $\gamma = \gamma \cdot \bP$ for the cutset Markov chain in Fig.\,\ref{fig:cutset_MC}: 
	$\gamma = \tuple{\nicefrac{48}{121}\ \nicefrac{18}{121}\ \nicefrac{40}{121}\ \nicefrac{15}{121}}$.
	As in this case the cutset $\cC = \{X,Y\}$ equals the set of all nodes \cV, we have $\Extend(\cB,\cC,\gamma) = \gamma$ and thus $\semantics{\cB}_{ \Cut\text{-}\{X,Y\} } = \{\gamma\}$.
\end{example}

As shown by Lemma~\ref{lemma:convex_combi_BSCCs}, the behavior of the $\Next$ function is captured by multiplication with the transition matrix \bP.
Both the distributions in the limit average semantics and the long-run frequency distributions of the cutset Markov chain are defined in terms of a Cesàro limit, the former over the sequence of distributions obtained by repeated application of \Next, the latter by repeated multiplication with \bP.
Thus, both semantics are equivalent.

\begin{theorem}%
	\label{thm:cutset_semantics_limAvg}
	Let $\cB$ be a GBN.
	Then for any cutset $\cC$ of $\cB$ and initial distribution $\gamma_0 \in \Dist(\Asg(\cC))$, we have
	\[ 
		\mathit{MCS}(\cB,\cC,\gamma_0) \ = \ \mathit{LimAvg}(\cB,\cC,\gamma_0).
	\]
\end{theorem}

We know that $\mathit{Lim}(\cB,\cC,\gamma_0)$ is not defined for all initial distributions $\gamma_0$.
However, the set of all limits that do exist contains exactly the distributions admitted by the Markov chain and limit average semantics.
\begin{lemma}
	\label{lemma:mc_lim_limavg}
	Let $\cB$ be a GBN.
	Then for any cutset $\cC$ of $\cB$, we have
	\[ 
		\semantics{\cB}_\CutC 
		\ = \ 
		\semantics{\cB}_\LimAvgC 
		\ = \
		\semantics{\cB}_\LimC .
	\]
\end{lemma}

\subsection{Properties}\label{sec:limit_semantics_properties}

By the equivalences established in Theorem~\ref{thm:cutset_semantics_limAvg} and Lemma~\ref{lemma:mc_lim_limavg},
we gain profound insights about the limit and limit average distributions 
by Markov chain analysis.
As every finite-state Markov chain has at least one stationary distribution,
it immediately follows that $\semantics{\cB}_\CutC$---and thus $\semantics{\cB}_\LimAvgC$ and $\semantics{\cB}_\LimC$---is always non-empty.
Further, if the cutset Markov chain is \emph{irreducible}, i.e., the graph is strongly connected, 
the stationary distribution is unique and $\semantics{\cB}_\CutC$ is a singleton.
The existence of the limit semantics for a given initial distribution $\gamma_0$ hinges on the \emph{periodicity} of the cutset Markov chain.
\begin{example}
	We return to Example~\ref{ex:no_limit} and construct the cutset Markov chain $\CutsetMC{\cB}{\cC} = \tuple{\Asg(\cC), \bP}$ for the (implicitly used) cutset $\cC = \{X,Y\}$:
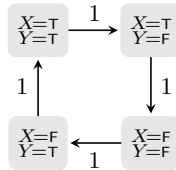
 	\begin{center}
 	\begin{tikzpicture}
 		\node (AB)   at (0.0, 0.0) [MCstate] {$\substack{ X\!{=}\T \\ Y\!{=}\T }$};
 		\node (AnB)  at (1.5, 0.0) [MCstate] {$\substack{ X\!{=}\T \\ Y\!{=}\F }$};
 		\node (nAB)  at (0.0,-1.5) [MCstate] {$\substack{ X\!{=}\F \\ Y\!{=}\T }$};
 		\node (nAnB) at (1.5,-1.5) [MCstate] {$\substack{ X\!{=}\F \\ Y\!{=}\F }$};
 		\draw[myarrow] (AB)   -> node[label, above]{$1$} (AnB);
 		\draw[myarrow] (AnB)  -> node[label, right]{$1$} (nAnB);
 		\draw[myarrow] (nAnB) -> node[label, below]{$1$} (nAB);
 		\draw[myarrow] (nAB)  -> node[label, left]{$1$} (AB);
 	\end{tikzpicture}
 	\end{center}
 	The chain is strongly connected and has a period of length $4$, which explains the observed behavior that for any initial distribution $\gamma_0$, we got the sequence 
 	\[
 		\gamma_0, \gamma_1, \gamma_2, \gamma_3, \gamma_0, \gamma_1, \dots
 	\]
	This sequence obviously converges only for initial distributions that are stationary, i.e., if we have $\gamma_0 = \gamma_0 \cdot \bP$.
\end{example}

The following lemma summarizes the implications that can be drawn from close inspection of the cutset Markov chain.

\begin{lemma}[Cardinality]\label{lemma:cardinality}
	Let \cB be a GBN with cutset \cC and cutset Markov chain $\CutsetMC{\cB}{\cC} = \tuple{\Asg(\cC), \bP}$.
	Further, let $k > 0$ denote the number of bottom SCCs $\cD_1,\dots,\cD_k$ of $\CutsetMC{\cB}{\cC}$.
	Then
	\begin{enumerate}
	\item the cardinality of the cutset Markov chain semantics is given by
	\[
	  \bigl| \semantics{\cB}_\CutC \bigr|
	  \ = \ 
	  \left\{
	  \begin{array}{cl}
	  	1      & \ \text{ if} \ \ k = 1, \\
	   	\infty & \ \text{ if} \ \ k > 1;
	  \end{array}
	  \right.
	\]
	\item
		$\mathit{Lim}(\cB,\cC,\gamma_0)$ is defined for all $\gamma_0 \in \Dist(\Asg(\cC))$ if all $\cD_i$ are aperiodic;
	\item
		$\mathit{Lim}(\cB,\cC,\gamma)$ is only defined for stationary distributions $\gamma$ with $\gamma = \gamma \cdot \bP$ if $\cD_i$ is periodic for any $1 \leqslant i \leqslant k$.
	\end{enumerate}
\end{lemma}

A handy sufficient (albeit not necessary) criterion for both aperiodicity 
and the existence of a single bottom SCC in the cutset Markov chain
is the absence of zero and one entries in the CPTs and the initial distribution of a GBN.
\begin{definition}[Smooth GBNs]\label{def:smooth}
 	A GBN $\cB = \tuple{ \cG, \cP, \init }$ is called \emph{smooth} iff all CPT entries as given by \cP and all values in $\init$ are in the open interval $]0,1[$.
\end{definition}

\begin{lemma}\label{lemma:smooth_complete}
 	Let \cB be a smooth GBN and $\cC$ a cutset of $\cB$.
 	Then the graph of the cutset Markov chain $\CutsetMC{\cB}{\cC}$ is a complete digraph.
\end{lemma}

\begin{corollary}\label{col:smooth_limit}
 	The limit semantics of a smooth GBN \cB is a singleton for every cutset \cC of \cB and $\mathit{Lim}(\cB,\cC,\gamma_0)$ is defined for all $\gamma_0 \in Dist(\Asg(\cC))$.
\end{corollary}

As noted in \cite{koller_probabilistic_2009}, one rarely needs to assign a probability of zero (or, conversely, of one) in real-world applications; and doing so in cases where some event is extremely unlikely but not impossible is a common modeling error.
This observation gives reason to expect that most GBNs encountered in practice are smooth, and their semantics is thus, in a sense, well-behaved.

\subsection{Relation to Constraints Semantics}
We take a closer look at how the cutset semantics relates to the CPT-consistency semantics defined 
in Section \ref{sec:constraints}. CPTs of nodes outside cutsets remain unaffected in 
the dissected BNs from which the Markov chain semantics is computed.
Since there are cyclic GBNs for which no CPT-consistent distribution exists 
(cf.\ Example~\ref{ex:cpt-semantics}) while Markov chain semantics always yields at least 
one solution due to Lemma~\ref{lemma:cardinality}, it cannot be expected that cutset nodes are necessarily CPT-consistent. %
However, they are always weakly CPT-consistent.

\begin{lemma}\label{lem:consistency-cutset}
	Let \cB be a GBN over nodes $\cV$, $\cC \subseteq \cV$ a cutset for $\cB$, and $\mu \in \CUTsem{\cB}{\cC}$.
	Then $\mu$ is strongly CPT-consistent for all nodes in $\cV{\setminus}\cC$ and weakly 
	CPT-consistent for the nodes in $\cC$.
\end{lemma}
The lemma shows a way to find fully CPT consistent distributions:
Consider there is a distribution $\mu \in \semantics{\cB}_\CutC \cap \semantics{\cB}_{\Cut\text{-}\cD}$ for two disjoint cutsets $\cC$ and $\cD$.
Then by Lemma~\ref{lem:consistency-cutset} the nodes in $\cV\setminus\cC$ and $\cV\setminus\cD$ are CPT consistent, so in fact $\mu$ is CPT consistent.
In general, we get the following result.
\begin{lemma}\label{lemma:cutset_MC_sem_CPT}
	Let \cB be a GBN over nodes \cV and $\cC_1, \dots, \cC_k$ cutsets of \cB s.t.\ for each node $X\!\in\cV$ there is an $i \in \{1,\dots,k\}$ with $X\!\notin \cC_i$.
	Then 
	\[
	\bigcap_{0 \leqslant i \leqslant k} 
		\semantics{\cB}_{\Cut\text{-}\cC_i}
		\ \subseteq \
		\semantics{\cB}_\CPT.
	\]
\end{lemma}

We take a look at which independencies are necessarily observed by the distributions in $\CUTsem{\cB}{\cC}$.
Let $\gamma \in \Dist(\Asg(\cC))$ be the cutset distribution and let $\cG[\cC]$ denote the graph of $\dissectedBN{\cB}{\cC}{\gamma}$ restricted to the nodes in $\cV$ such that the cutset nodes in $\cC$ are initial.
Then by Lemma~\ref{lemma:d-sep_init}, the $d$-separation independencies of the closure of $\cG[\cC]$ hold in all distributions $\mu \in \CUTsem{\cB}{\cC}$, i.e., $\dsep\bigl(\initCl(\cG[\cC])\bigr) \subseteq \Indep(\mu)$.
The next lemma states that any \CPT-consistent distribution that satisfies these independence constraints for some cutset $\cC$ also belongs to $\CUTsem{\cB}{\cC}$.

\begin{lemma}\label{lem:cpt-in-cutset}
	Let $\cB$ be a GBN with cutset $\cC$ and $\cI_\cC = \initCl(\cG[\cC])$.
	Then we have 
	\[
		\semantics{\cB}_{\CPT\text{-}\cI_\cC}
		\ \subseteq \
		\semantics{\cB}_\CutC.
	\]
\end{lemma}

Combining Lemma~\ref{lemma:cutset_MC_sem_CPT} and Lemma~\ref{lem:cpt-in-cutset}
yields the following equivalence.
\begin{corollary}\label{col:MC_CPTI}
	For a GBN \cB with cutsets $\cC_1,\dots,\cC_k$ as in Lemma~\ref{lemma:cutset_MC_sem_CPT}
	and the independence set $\cI = \bigcup_{0 \leqslant i \leqslant k} \initCl(\cG[\cC_i])$, we have
	\[
		\bigcap_{0 \leqslant i \leqslant k}
		\semantics{\cB}_{\Cut\text{-}\cC_i}
		\ = \
		\semantics{\cB}_{\CPTI} .
	\]
\end{corollary}

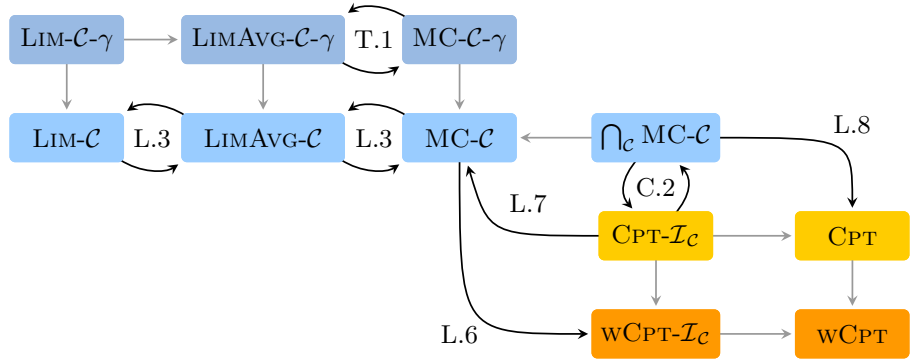
\begin{figure*}[t]
	\centering
	\resizebox{\textwidth}{!}{
	\begin{tikzpicture}[yscale=.8,xscale=1.2]

\node(LimCg)    at (0.0, 4.5) [mynode, fill=yourBlue] {$\LimC\text{-}\gamma$};
\node(LimAvgCg) at (2.0, 4.5) [mynode, fill=yourBlue, minimum width=20mm] {$\LimAvgC\text{-}\gamma$};
\node(MCCg)     at (4.0, 4.5) [mynode, fill=yourBlue] {$\CutC\text{-}\gamma$};
\node(LimC)     at (0.0, 3.0) [mynode, fill=myBlue] {\LimC};
\node(LimAvgC)  at (2.0, 3.0) [mynode, fill=myBlue, minimum width=20mm] {\LimAvgC};
\node(MCC)      at (4.0, 3.0) [mynode, fill=myBlue] {\CutC};
\node(CapMCC)   at (6.0, 3.0) [mynode, fill=myBlue] {$\bigcap_\cC\CutC$};
\node(CPTI)     at (6.0, 1.5) [mynode, fill=myYellow] {$\CPTI_\cC$};
\node(CPT)      at (8.0, 1.5) [mynode, fill=myYellow] {\CPT};
\node(wCPTI)    at (6.0, 0.0) [mynode, fill=myOrange] {$\wCPTI_\cC$};
\node(wCPT)     at (8.0, 0.0) [mynode, fill=myOrange] {\wCPT};

\draw[light arrow] (LimCg) -> (LimC);
\draw[light arrow] (LimAvgCg) -> (LimAvgC);
\draw[light arrow] (MCCg) -> (MCC);
\draw[light arrow] (CapMCC) -> (MCC);
\draw[light arrow] (CPTI) -> (CPT);
\draw[light arrow] (CPTI) -> (wCPTI);
\draw[light arrow] (wCPTI) -> (wCPT);
\draw[light arrow] (CPT) -> (wCPT);

\draw[light arrow] (LimCg) to (LimAvgCg);
\draw[long myarrow, bend right=45] (MCCg.north west) to (LimAvgCg.north east);
\draw[long myarrow, bend right=45] (LimAvgCg.south east) to (MCCg.south west);

\draw[long myarrow, bend right=45] (LimAvgC.north west) to (LimC.north east);
\draw[long myarrow, bend right=45] (LimC.south east) to (LimAvgC.south west);
\draw[long myarrow, bend right=45] (MCC.north west) to (LimAvgC.north east);
\draw[long myarrow, bend right=45] (LimAvgC.south east) to (MCC.south west);

\draw[long myarrow, bend right=30] (CapMCC) to (CPTI);
\draw[long myarrow, bend right=30] (CPTI) to (CapMCC);
\draw[long myarrow] (CapMCC) .. controls (8.0, 3.0) .. (CPT);
\draw[long myarrow] (CPTI)   .. controls (4.3, 1.5) .. (MCC);
\draw[long myarrow] (MCC)    .. controls (4.0, 0.0) .. (wCPTI);

\node at (6.0, 2.25) [] {C.\ref{col:MC_CPTI}};
\node at (4.0, 0.0) [] {L.\ref{lem:consistency-cutset}};
\node at (4.7, 2.0) [] {L.\ref{lemma:cutset_MC_sem_CPT}};
\node at (8.0, 3.2) [] {L.\ref{lem:cpt-in-cutset}};
\node at (0.87, 3.0) [] {L.\ref{lemma:mc_lim_limavg}};
\node at (3.13, 3.0) [] {L.\ref{lemma:mc_lim_limavg}};
\node at (3.13, 4.5) [] {T.\ref{thm:cutset_semantics_limAvg}};

\node(LimCg)    at (0.0, 4.5) [mynode, fill=yourBlue] {$\LimC\text{-}\gamma$};
\node(LimAvgCg) at (2.0, 4.5) [mynode, fill=yourBlue, minimum width=20mm] {$\LimAvgC\text{-}\gamma$};
\node(MCCg)     at (4.0, 4.5) [mynode, fill=yourBlue] {$\CutC\text{-}\gamma$};
\node(LimC)     at (0.0, 3.0) [mynode, fill=myBlue] {\LimC};
\node(LimAvgC)  at (2.0, 3.0) [mynode, fill=myBlue, minimum width=20mm] {\LimAvgC};
\node(MCC)      at (4.0, 3.0) [mynode, fill=myBlue] {\CutC};
\node(CapMCC)   at (6.0, 3.0) [mynode, fill=myBlue] {$\bigcap_\cC\CutC$};
\node(CPTI)     at (6.0, 1.5) [mynode, fill=myYellow] {$\CPTI_\cC$};
\node(CPT)      at (8.0, 1.5) [mynode, fill=myYellow] {\CPT};
\node(wCPTI)    at (6.0, 0.0) [mynode, fill=myOrange] {$\wCPTI_\cC$};
\node(wCPT)     at (8.0, 0.0) [mynode, fill=myOrange] {\wCPT};

\end{tikzpicture}}
	\caption[]{Relations between different variations of limit, limit average, and Markov chain semantics (blue) as well as strong and weak CPT-consistency semantics (yellow resp.~orange)}
	\label{fig:semantics_overview}
\end{figure*}

\subsection{Overview}

Fig.\,\ref{fig:semantics_overview} gives an overview of the relations between all proposed semantics.
Boxes represent the set of distributions induced by the respective semantics and arrows stand for set inclusion.
For the non-trivial inclusions the arrows are annotated with the respective lemma or theorem.
As an example, $\CPT{\to}\wCPT$ states that $\semantics{\cB}_\CPT \subseteq \semantics{\cB}_\wCPT$ holds for all GBNs $\cB$.
The three semantics in the top row parameterized with a cutset $\cC$ and a distribution $\gamma$ stand for the singleton set containing the respective function applied to $\gamma$, i.e., $\semantics{\cB}_{\LimC\text{-}\gamma} = \{\mathit{Lim}(\cB,\cC,\gamma)\}$.
$\bigcap_\cC\CutC$ stands for the intersection of the Markov chain semantics for various cutsets as in Lemma~\ref{lemma:cutset_MC_sem_CPT}, and the incoming arrow from $\CPT\text{-}\cI_\cC$ holds for the set of independencies $\cI_\cC$ as in Lemma~\ref{lem:cpt-in-cutset}.
\section{Related Work}
\label{sec:related_work}

That cycles in a BN might be unavoidable when learning its structure is well known for more 
than 30 years~\cite{LauSpi88a,Pea90a}. 
During the learning process of BNs, cycles might even be favorable as demonstrated in the 
context of \emph{gene regulatory networks} where cyclic structures induce monotonic 
scores~\cite{wiecek2020structure}. That work only discusses learning algorithms, 
but does not deal with evaluating the joint distribution of the resulting cyclic BNs.
In most applications, however, cycles have been seen as a phenomenon to be avoided to ease the computation of the joint distribution in BNs. 
By an example BN comprising a single isolated cycle, \cite{TulNik05a} showed that reversing or 
removing edges to avoid cycles may reduce the solution space from infinitely many joint 
distributions that are (weakly) consistent with the CPTs to a single one.
In this setting, our results on weak CPT-semantics also provide 
that \wCPT cannot express conditions on the relation of variables 
like implications or mutual exclusion.
This is rooted in the fact that the solution space of weak CPT-semantics always contains 
at least one full joint distribution with pairwise independent variables.
An example where reversing edges %
led to satisfactory results has been considered in \cite{CasGutHad98a}, investigating
the impact of reinforced defects by steel corrosion in concrete structures.

Unfolding cycles up to a bounded depth has been applied 
in the setting of a robotic sensor system by \cite{CasFerGar21a}.
In their use case, only cycles of length two may appear, and only the nodes appearing on 
the cycles are implicitly used as cutset for the unfolding.
In \cite{Klopotek2006}, the set of all nodes is used for unfolding (corresponding to a 
cutset $\cC = \cV$ in our setting) and subsequent limit construction, but restricted to 
cases where the limit exists.\\

There have been numerous variants of BNs that explicitly or implicitly address cyclic
dependencies.
\emph{Dynamic Bayesian networks} (DBNs)~\cite{Mur02a} extend BNs by an explicit notion
of discrete time steps that could break cycles through timed ordering of random variables.
Cycles in BNs could be translated to the DBN formalism
by introducing a notion of time, e.g., following \cite{Klopotek2006}.
Our cutset approach is orthogonal, choosing a time-abstract view on cycles
and treating them as stabilizing feedback loops.
Learning DBNs requires ``relatively large time-series data''~\cite{wiecek2020structure}
and thus, may be computationally demanding.  
In \cite{Motzek2017} \emph{activator random variables} break cycles in DBNs
to circumvent spurious results in DBN reasoning when infinitesimal small time 
steps would be required.

\emph{Causal BNs}~\cite{Pearl09} are BNs that impose a meaning on the direction of 
an edge in terms of causal dependency. 
Several approaches have been proposed to extend causal BNs for modeling feedback loops.
In \cite{poole_cyclic_2013}, an equilibrium semantics is sketched that is similar to our Markov chain semantics, albeit based on variable oderings rather than cutsets.
Determining independence relations, Markov properties, and
joint distributions are central problems addressed for cyclic causal 
BNs \cite{CasFerGar21a,forr2017markov,Nea00a,PeaDec96a,Spi95a}.
Markov properties and joint distributions for extended versions
of causal BNs have been considered recently, e.g., in
\emph{directed graphs with hyperedges} (HEDGes) \cite{forr2017markov} and
\emph{cyclic structural causal models} (SCMs) \cite{CasFerGar21a}.
Besides others, they show that in presence of cycles, there might be multiple solutions 
for a joint distribution or even no solution at all~\cite{Hal00a}.
While we consider all random variables to be observable, the latter approaches 
focus on models with latent variables. Further, 
while our focus in this paper is not on causality,
our approach is surely also applicable to causal BNs with cycles.

\emph{Recursive relational Bayesian networks} (RRBNs) \cite{jaeger2001} allow representing probabilistic relational models where the random variables are given by relations over varying domains. 
The resulting first-order dependencies can become quite complex and may contain cycles, though semantics are given only for the acyclic cases by the construction of corresponding standard BNs.

\emph{Bayesian attack graphs} (BAGs)~\cite{LiuMan05a} are popular to model and reason 
about security vulnerabilities in computer networks. 
Learned graphs and thus their BN semantics frequently contain cycles, 
e.g., when using the tool \textsc{MulVAL}~\cite{OuGovApp05a}. 
In \cite{Singhal2017}, ``handling cycles correctly'' is identified as ``a key challenge'' in 
security risk analysis.
Resolution methods for cyclic patterns in BAGs \cite{AguBetBla16a,DoyKot17a,matthews2020cyclic,WanIslLon08a}
are mainly based on context-specific security considerations, e.g., to break cycles by removing edges.
The semantic foundations for cyclic BNs laid in this paper do not require graph manipulations
and decouple the probability theoretic basis from context-specific properties.

\section{Conclusion}

This paper has developed a foundational perspective on the semantics of cycles in Bayesian networks.
Constraint-based semantics provide a conservative extension of the standard BN semantics to the cyclic setting. 
While conceptually important, 
their practical use is limited by the fact that for many GBNs, the induced constraint system is unsatisfiable.
On the other hand, the two introduced limit semantics echo in an abstract and formal way what practitioners have been devising across a 
manifold of domain-specific situations. 
In this abstract perspective, cutsets are the ingredients that 
enable a controlled decoupling of dependencies.
The appropriate choice of cutsets is where, in our view, domain-specific 
knowledge is confined to enter the picture.
Utilizing the constructively defined Markov chain semantics, we established key results relating and demarcating the different semantic notions and showed that for the ubiquitous class of smooth GBNs a unique full joint distribution always exists.

\bibliographystyle{splncs04}
\bibliography{main}
\clearpage
\appendix
\section{Appendix}

The appendix contains the proofs omitted from the body of the submission ``On the Foundations of Cycles in Bayesian Networks'' due to space constraints.

\addtocounter{lemma}{-8}
\addtocounter{corollary}{-2}
\begin{lemma}%
	Let $\cB_\nocycle = \tuple{\cG,\cP,\init}$ be an acyclic GBN.
	Then
	\[
		\dsep\bigl( \initCl(\cG) \bigr) 
		\ \subseteq \ 
		\Indep\bigl( \distBN(\cB_\nocycle) \bigr) .
	\]
\end{lemma}
\begin{proof}
	The idea is to show that the dependencies of every possible BN structure for the initial distribution $\init$ are covered by the closure operation.
	Let the graph $\cG_\init^\star = \tuple{\Init(\cG), \cE^\star}$ be a DAG that is an I-map for $\init$, i.e., $\dsep(\cG_\init^\star) \subseteq \Indep(\init)$.
	Then $\init$ factorizes according to $\cG_\init^\star$, that is for every assignment $b \in \Asg(\Init(\cG))$, we have
	\[
		\init(b) \ = \
		\prod_{\mathclap{X \in \Init(\cG)}} \ 
		\init\bigl(b_X \mid b_{\Pre^{\cG_\init^\star}(X)}\bigr) .
	\]
	Now consider the BN $\cB_\nocycle^\star$ with graph $\cG^\star = \tuple{\cV, \cE \cup \cE^\star}$ where we add the edges of $\cG_\init^\star$ to $\cG$.
	The CPTs for the nodes in $\cV \setminus \Init(\cG)$ are given by $\cP$ whereas the new CPTs (according to the structure in $\cG_\init^\star$) for the nodes in $\Init(\cG)$ are derived from $\init$.
	Then for every assignment $c \in \Asg(\cV)$:
	\begin{align*}
		\distBN(\cB_\nocycle^\star)(c)
		\ &= \
		\prod_{\mathclap{ X \in \cV }} 
		\ \CPTable\bigl(c_X \mid c_{\Pre(X)}\bigr) \\
		\ &= \
		\prod_{\mathclap{X \in \Init(\cG)}} \ 
		\init\bigl(c_X \mid c_{\Pre^{\cG_\init^\star}(X)}\bigr)
		\ \cdot \ 
		\prod_{\mathclap{ X \in \cV \setminus \Init(\cG) }} 
		\ \CPTable\bigl(c_X \mid c_{\Pre(X)}\bigr) \\
		\ &= \
		\init\bigl(c_{\Init(\cG)}\bigr)
		\ \cdot \ 
		\prod_{\mathclap{ X \in \cV \setminus \Init(\cG) }} 
		\ \CPTable\bigl(c_X \mid c_{\Pre(X)}\bigr) \\
		\ &= \
		\distBN(\cB_\nocycle)(c).
	\end{align*}
	As $\cB_\nocycle^\star$ is a regular BN without an initial distribution, we have $\dsep(\cG^\star) \subseteq \Indep(\distBN(\cB_\nocycle^\star))$.

	We proceed to show $\dsep(\initCl(\cG)) \subseteq \dsep(\cG^\star)$.
	Let $(X \perp Y \mid \cZ) \in \dsep(\initCl(\cG))$.
	Then each path from $X$ to $Y$ in $\initCl(\cG)$ is blocked by the nodes in $\cZ$. 
	As $\initCl(\cG)$ contains all possible edges between the nodes $\Init(\cG)$ but $\cG^\star$ only a subset thereof, it is clear that each path in $\cG^\star$ also exists in $\initCl(\cG)$.
	Thus, there cannot be an unblocked path from $X$ to $Y$ given $\cZ$ in $\cG^\star$ either, so $(X \perp Y \mid \cZ) \in \dsep(\cG^\star)$.
	Altogether, we have
	\[
		\dsep\bigl( \initCl(\cG) \bigr) 
		\subseteq
		\dsep(\cG^\star)
		\subseteq
		\Indep\bigl( \distBN(\cB_\nocycle^\star) \bigr)
		=
		\Indep\bigl( \distBN(\cB_\nocycle) \bigr). 
	\]\qed
\end{proof}

\begin{lemma}%
	Let $\cB$ be a GBN with cutset $\cC$, cutset distribution $\gamma \in \Dist(\Asg(\cC))$, and $\cM = \tuple{\Asg(\cC), \bP}$ the cutset Markov chain $\CutsetMC{\cB}{\cC}$.
	Then the following statements are equivalent:
 	\begin{enumerate}[label=(\alph*),leftmargin=*] %
 	\item 
 		$\gamma =\gamma \cdot \bP$. %
	\item
		There exists $\gamma_0\in \Dist(\Asg(\cC))$ such that for $\gamma_{i+1} =\gamma_i \cdot \bP$, we have
		\[
			\gamma\ = \ \lim\limits_{n\to \infty}\  
			\frac{1}{n{+}1}\ \sum_{i=0}^n \gamma_i .
		\]
	\item\label{item:convex} 
		$\gamma$ belongs to the convex hull of the long-run
		frequency distributions $\LongRunFreq_\cD$ of the bottom SCCs $\cD$ of $\cM$.
 	\item 
 		$\gamma = \Next(\cB, \cC, \gamma)|_\cC$.
  \end{enumerate}
\end{lemma}
\begin{proof}
	(a) $\Longrightarrow$ (b):
	If we have $\gamma =\gamma \cdot \bP$, then statement (b) is obtained by considering
	$\gamma_0 = \gamma$, as then $\gamma_i = \gamma$ for all $i$.
	
	(b) $\Longrightarrow$ (c): 
	The proof of the implication relies on the following standard facts about finite-state Markov chains.
	Given a BSCC $\cD$ and 
	an arbitrary distribution $\nu_0 \in \Dist(\Asg(\cD))$, 
	the distribution $\LongRunFreq_\cD$ agrees with the Cesàro limit
	of the sequence $(\nu_i)_{i \geqslant 0}$ where
	$\nu_{i+1} = \nu_i \cdot \bP_{\cD}$ and 
	$\bP_{\cD}$ denotes the restriction of $\bP$ to assignments on $\cD$.
	That is, \vspace{-0.2cm}
	\[
		\LongRunFreq_\cD = \lim\limits_{n\to \infty} 
		\frac{1}{n{+}1}\sum\limits_{i=0}^n \nu_i .
	\]
	
	Vice versa, for $\gamma_0 \in \Dist(\Asg(\cC))$ and
	$\gamma_{i+1} = \gamma_i \cdot \bP$, then
	the Cesàro limit $\gamma$
	of the sequence $(\gamma_i)_{i \geqslant 0}$ has the form
	\[
		\gamma = \sum_{\cD} \lambda(\cD) \cdot \LongRunFreq_\cD
	\]
	where $\cD$ ranges over all BSCCs of $\cM$,
	$\lambda(\cD)$ is the probability for reaching $\cD$ in
	$\cM$ with the initial distribution $\gamma_0$,
	and all vectors $\LongRunFreq_\cD$ are padded with zero entries to range over the whole state space. 
	In particular, $\gamma$ is a convex combination of the distributions
	$\LongRunFreq_\cD$ as $0 \leqslant \lambda(\cD) \leqslant 1$ and
	$\sum_{\cD} \lambda(\cD) = 1$ 
	(because every finite-state Markov chain almost surely reaches a BSCC).

	(c) $\Longrightarrow$ (a): Suppose
	$\gamma = \sum_{\cD} \lambda(\cD) \cdot \LongRunFreq_\cD$ where
	$0\leqslant \lambda(\cD) \leqslant 1$, $\sum_{\cD} \lambda(\cD) = 1$, and each $\LongRunFreq_\cD$ is padded appropriately as before. 
	Then:
	\[
		\gamma \cdot \bP  
		\ = \ 
		\sum_{\cD} \lambda(\cD) \cdot \LongRunFreq_\cD \cdot \bP
		\ = \ 
		\sum_{\cD} \lambda(\cD) \cdot \LongRunFreq_\cD 
		\ = \ \gamma
	\]
	where we use the fact that $\LongRunFreq_\cD = \LongRunFreq_\cD \cdot \bP$.

	(a) $\Longleftrightarrow$ (d):
	Because $\gamma$ can be represented as convex combination of Dirac distributions as $\gamma = \sum_{c \in \Asg(\cC)} \gamma(c) \cdot \Dirac(c)$, we know:
	\[
		\Next(\cB, \cC, \gamma)
		\ = \ 
		\sum_{\mathclap{ c \in \Asg(\cC) }} \
		\gamma(c) 
		\cdot 
		\Next\bigr(\cB, \cC, \Dirac(c)\bigl).
	\]
	As $\bP(c,b) = \Next\bigl(\cB, \cC, \Dirac(c)\bigr)(b)$ for any assignment $b \in \Asg(\cC)$, and assuming $\gamma = \gamma \cdot \bP$, we get 
	\[
		\Next(\cB, \cC, \gamma)(b)
		\ = \ 
		\sum_{\mathclap{ c \in \Asg(\cC) }} \
		\gamma(c) 
		\cdot 
		\bP(c,b)
		\ = \
		(\gamma \cdot \bP)(b)
		\ = \
		\gamma(b).
	\]
	Conversely, assuming $\Next(\cB,\cC,\gamma)|_\cC = \gamma$, we yield $\gamma = \gamma \cdot \bP$.
	\qed
\end{proof}

\begin{lemma}%
	Let $\cB$ be a GBN.
	Then for any cutset $\cC$ of $\cB$, we have
	\[ 
		\semantics{\cB}_\CutC 
		\ = \ 
		\semantics{\cB}_\LimAvgC 
		\ = \
		\semantics{\cB}_\LimC .
	\]
\end{lemma}
\begin{proof}
	We have $\semantics{\cB}_\CutC = \semantics{\cB}_\LimAvgC$ by Theorem~\ref{thm:cutset_semantics_limAvg} and know $\semantics{\cB}_\LimC \subseteq \semantics{\cB}_\LimAvgC$, so it remains to show $\semantics{\cB}_\CutC \subseteq \semantics{\cB}_\LimC$.
	Let $\mu \in \semantics{\cB}_\CutC$.
	Then there exists a cutset distribution $\gamma$ s.t.\ $\mu = \Extend(\cB,\cC,\gamma)$.
	We need to show there exists an initial distribution $\gamma_0 \in \Dist(Asg(\cC))$ such that $\gamma = \lim_{n \to \infty} \gamma_i$ 
	where $\gamma_{i+1} = \Next(\cB, \cC, \gamma_i)|_\cC$.
	Let us choose $\gamma_0 = \gamma$. 
	Then we know $\gamma_0 = \Next(\cB,\cC,\gamma_0)|_\cC$
	by Lemma~\ref{lemma:convex_combi_BSCCs}, so
	$\gamma_i = \gamma_0$ for all $i\in\Nat$.
	Thus, $\gamma = \lim_{i \to \infty} \gamma_i$ and therefore $\mu \in \semantics{\cB}_\LimC$.
	\qed
\end{proof}

\begin{lemma}[Cardinality]%
	Let \cB be a GBN with cutset \cC and cutset Markov chain $\CutsetMC{\cB}{\cC} = \tuple{\Asg(\cC), \bP}$.
	Further, let $k > 0$ denote the number of bottom SCCs $\cD_1,\dots,\cD_k$ of $\CutsetMC{\cB}{\cC}$.
	Then
	\begin{enumerate}
	\item the cardinality of the cutset Markov chain semantics is given by
	\[
	  \bigl| \semantics{\cB}_\CutC \bigr|
	  \ = \ 
	  \left\{
	  \begin{array}{cl}
	  	1      & \ \text{ if} \ \ k = 1, \\
	   	\infty & \ \text{ if} \ \ k > 1;
	  \end{array}
	  \right.
	\]
	\item
		$\mathit{Lim}(\cB,\cC,\gamma_0)$ is defined for all $\gamma_0 \in \Dist(\Asg(\cC))$ if all $\cD_i$ are aperiodic;
	\item
		$\mathit{Lim}(\cB,\cC,\gamma)$ is only defined for stationary distributions $\gamma$ with $\gamma = \gamma \cdot \bP$ if $\cD_i$ is periodic for any $1 \leqslant i \leqslant k$.
	\end{enumerate}
\end{lemma}
\begin{proof}
    (1.) By \Cref{lemma:convex_combi_BSCCs}, every cutset distribution with $\gamma = \gamma \cdot \bP$ is a convex combination of the steady-state distributions for the BSCCs. Thus, for $k=1$ a unique distribution $\gamma$ exists, whereas for $k>1$, there are infinitely many real-valued distributions in the convex hull.

    \noindent
    (2.) A Markov chain is aperiodic if all its BSCCs are aperiodic. 
    Aperiodicity suffices for the limit $\lim_{n \to \infty} \gamma_n$ with $\gamma_{n+1} = \gamma_n \cdot \bP$ to exist for every $\gamma_0$.
    Then $\lim_{n \to \infty} \gamma'_n$ with $\gamma'_{n+1} = \Next(\cB, \cC, \gamma'_n)|_\cC$ exists as well by \Cref{lemma:convex_combi_BSCCs}.

    \noindent
    (3.) Assume some BSCC $\cD$ is periodic with a period of $p$. 
    Then, for any $\gamma_0 \in \Dist(\Asg(\cC))$, $\gamma_{n+1} = \gamma_n \cdot \bP$, and $\nu_n = \gamma_n|_\cD$, we have $\nu_{p \cdot n} = \nu_n$.
    Now consider $\gamma_0$ and $\gamma_1 = \gamma_0 \cdot \bP$. 
    If $\gamma_0 = \gamma_1$, then $\gamma_0 = \gamma_n$ for all $n\in\Nat$ and $\gamma_0 = \lim_{n \to \infty} \gamma_n$ holds.
    Otherwise, if $\gamma_0 \neq \gamma_1$, the following non-convergent sequence exists:
    \[
        \nu_0, \nu_1, \dots, \nu_p, \nu_{p+1}, \dots, \nu_{2p}, \nu_{2p+1}, \dots
    \]
    Then $\lim_{n \to \infty} \gamma_n$ cannot converge either, so $\Lim(\cB,\cC,\gamma_0)$ is undefined.
    \qed
\end{proof}

\begin{lemma}%
    Let \cB be a smooth GBN and $\cC$ a cutset of $\cB$.
    Then the graph of the cutset Markov chain $\CutsetMC{\cB}{\cC}$ is a complete digraph.
\end{lemma}
\begin{proof}
    The graph of $\CutsetMC{\cB}{\cC} = \tuple{\Asg(\cC), \bP}$ is a complete digraph iff each entry in $\bP$ is positive.
    Thus, for each two assignments $b, c \in \Asg(\cC)$, we need to show $\bP(b,c) > 0$.
    Let $\cB_b = \dissectedBN{\cB}{\cC}{\Dirac(b)}$.
    Then from Definition \ref{def:cutsetMC}, we have
   \begin{align*}
       \bP(b,c) &= \nextdisBN{\Dirac(b)}{\cB}{\cC}(c) \\
                &= \BNsem{\cB_b}(c').
   \end{align*}
    The probability $\BNsem{\cB_b}(c')$ is given by the sum over all full assignments $v \in \Asg(\cV)$ that agree with $c'$ on the assignment of the cutset node copies $\cC'$.
    Further, the sum can be partitioned into those $v$ that agree with assignment $b$ on \cC and those that do not:
    \[ 
        \BNsem{\cB_b}(c') \ = \
            \sum_{\mathclap{ \substack{ v \in \Asg(\cV) \\ s.t.\ c'\subset v,\ b\subset v}} } 
            \ \BNsem{\cB_b}(v) \ + \ 
            \sum_{\mathclap{ \substack{ v \in \Asg(\cV) \\ s.t.\ c\subset v,\ b \not\subset v}} }
            \ \BNsem{\cB_b}(v). 
    \]
    By the definition of the standard \BN-semantics, we have
   \[
        \BNsem{\cB_b}(v) \ = \
           \init\bigl( v_{\Init(\cG)} \bigr) \cdot 
              Dirac(b)(v_\cC) \cdot
              \prod_{\mathclap{ X \in \cV \setminus \cC }} \ 
              \CPTable\bigl( v_X \mid v_{ \Pre(X) } \bigr).
   \]
    Now consider the second sum in the previous equation where $b \not\subset v$. %
   For those assignments, $\Dirac(b)(v_\cC) = 0$ and thus the whole sum equals zero.
   For the first sum, we have $v_\cC = b$, so $\Dirac(b)(v_\cC) = 1$ and we only need to consider the product with $X\! \in \cV \setminus \cC$ and the initial distribution over $\Init(\cG)$.
    By the construction of $\cB_b$, the CPTs of all $X \in \cV \setminus \cC$ are the original CPTs from $\cB$, thus their entries all fall within the open interval $]0, 1[$ by the smoothness assumption of $\cB$.
   The same holds for the value $\init\bigl( v_{\Init(\cG)} \bigr)$.
    Thus, the whole product resides in $]0, 1[$ as well.
    Finally, note that the sum is non-empty as $\cC'$ and \cC are disjoint, so there exists at least one $v \in \Asg(\cV)$ 
    with $c \subset v$ and $b \subset v$.
    As a non-empty sum over values in $]0, 1[$ is necessarily positive, we have $\BNsem{\cB_b}(c') > 0$ and the claim follows.
    \qed
\end{proof}

\begin{corollary}%
    The limit semantics of a smooth GBN \cB is a singleton for every cutset \cC of \cB and $\mathit{Lim}(\cB,\cC,\gamma_0)$ is defined for all $\gamma_0 \in Dist(\Asg(\cC))$.
\end{corollary}
\begin{proof}
    Follows from Lemma~\ref{lemma:cardinality} and
    Lemma~\ref{lemma:smooth_complete} because every complete graph forms a single bottom SCC and is necessarily aperiodic.
    \qed
\end{proof}

\begin{lemma}%
	Let \cB be a GBN over nodes $\cV$, $\cC \subseteq \cV$ a cutset for $\cB$, and $\mu \in \CUTsem{\cB}{\cC}$.
	Then $\mu$ is strongly CPT-consistent for all nodes in $\cV{\setminus}\cC$ and weakly 
	CPT-consistent for the nodes in \cC.
\end{lemma}
\begin{proof}
	By definition, $\mu = \Extend(\cB, \cC, \gamma)$ for some
	$\gamma \in \Dist(\Asg(\cC))$ with $\gamma = \gamma \cdot \bP$.
	As $\Extend(\cB, \cC, \gamma)$ is the standard
	BN semantics for the acyclic BN $\dissectedBN{\cB}{\cC}{\gamma}$
	without the copies of the cutset nodes,
	CPT-consistency for the nodes in $\cV \setminus \cC$ follows
	directly from the CPT-consistency of the standard semantics for acyclic BNs.
   
	It remains to prove weak CPT-consistency for the cutset nodes.
	Let 
	$\delta = \BNsem{\dissectedBN{\cB}{\cC}{\gamma}} \in
	   \Dist(\Asg(\cV \cup \cC'))$. Thus, $\mu = \delta|_{\cV}$ and
	$\gamma = \delta|_{\cC}$.
	Then for each assignment $b\in \Asg(\cC)$, we have
	\[
	 \mu(b) \ = \ 
	 \gamma(b) \ = \ (\gamma \cdot \bP)(b) \ = \ 
	 \delta(b') 
	\]
	where $b'\in \Asg(\cC')$ is given by $b'(Y')=b(Y)$ for all $Y\!\in \cC$.
	In particular, for each $Y\!\in \cC$:
	\[
	 \mu(Y\!{=}\T) \ = \ 
	 \delta(Y'{=}\T) 
	\] 
	Let $D=\Asg(\Pre(Y))$ where $\Pre(\cdot)$ refers to the original \scGBN. 
	For $c\in \Asg(\cC)$, we write $D_c$ for
	the set of all assignments $d \in D$
	that comply with $c$ in the sense
	that if $Z\in \cC\cap \Pre(Y)$
	then $c(Z)=d(Z)$. In this case, $c$ and $d$ can be combined to an assignment
	for $\cC \cup \Pre(Y)$.
	Similarly, if $d\in D$, then the notation $\Asg_d(\cC)$ is used 
	for the set of assignments $c\in \Asg(\cC)$ that comply with $d$.
	Then:
	\begin{eqnarray*}
	  \delta(Y'{=}\T)  & = & 
	  \sum_{\mathclap{ c\in \Asg(\cC) }} \ \delta(Y'{=}\T \mid c) \cdot \mu(c)
	  \\
	  \\[-1ex]
	  & = & 
	 \sum_{\mathclap{ c \in \Asg(\cC) }} \ \ \ \ \sum_{d \in D_c}
			   \underbrace{\delta(Y'{=}\T \mid c, d)}_{\CPTable(Y{=}\T \mid d)}
		 \cdot \underbrace{\delta(d \mid c)}_{\mu(d \mid c)} 
		 \cdot \underbrace{\delta(c)}_{\mu(c)}
	  \\
	  \\[-1ex]
	  & = &
	  \sum_{d\in D}
		\CPTable(Y{=}\T \mid d) \cdot \
		\sum_{\mathclap{ c \in \Asg_d(\cC) }} \ \mu(d \mid c) \cdot \mu(c)
	  \\
	  \\[-1ex]
	  & = &
	  \sum_{d\in D}
		 \CPTable(Y{=}\T \mid d) \cdot \mu(d) .
	\end{eqnarray*}
	Putting everything together, we obtain: 
	\[
	  \mu(Y{=}\T) \ = \ 
	  \delta(Y'{=}\T) \ = \
	  \sum_{d\in D}
		 \CPTable(Y{=}\T \mid d) \cdot \mu(d) .
	\]
	Thus, $\mu$ is weakly CPT-consistent for $Y\!\in \cC$.
	\qed
\end{proof}

\begin{lemma}%
	Let \cB be a GBN over nodes \cV and $\cC_1, \dots, \cC_k$ cutsets of \cB s.t.\ for each node $X\!\in\cV$ there is an $i \in \{1,\dots,k\}$ with $X\!\notin \cC_i$.
	Then 
	\[
	\bigcap_{0 \leqslant i \leqslant k} 
		\semantics{\cB}_{\Cut\text{-}\cC_i}
		\ \subseteq \
		\semantics{\cB}_\CPT.
	\]
\end{lemma}
\begin{proof}
	We need to show CPT-consistency for every node under 
	$\mu \in \bigcap_i \semantics{\cB}_{\Cut\text{-}\cC_i}$.
	Let $X\!\in \cV$. Then we choose a cutset $\cC_i$ s.t.\ $X\!\notin \cC_i$ and CPT consistency follows from Lemma~\ref{lem:consistency-cutset}. \qed
\end{proof}

\begin{lemma}%
	Let $\cB$ be a GBN with cutset $\cC$ and $\cI_\cC = \dsep\bigl( \initCl(\initCl(\cG)[\cC]) \bigr)$.
	Then we have 
	\[
		\semantics{\cB}_{\CPT\text{-}\cI_\cC}
		\ \subseteq \
		\semantics{\cB}_\CutC.
	\]
\end{lemma}
\begin{proof}
	Let $\mu \in \semantics{\cB}_{\CPT\text{-}\cI_\cC}$ and $\gamma = \mu|_\cC$. The task is to show that
	$\gamma$ satisfies the fixed point equation 
	$\gamma = \gamma \cdot \bP$.
  
	The standard BN semantics 
	$\delta = \BNsem{\dissectedBN{\cB}{\cC}{\gamma}}$ 
	of the dissected BN is the unique distribution over
	$\Asg(\cV \cup \cC')$ that 
	\begin{itemize}
	\item
		is CPT-consistent w.r.t.\ the conditional probability tables in $\dissectedBN{\cB}{\cC}{\gamma}$,
	\item
		agrees with $\gamma$ when restricted to the assignments for $\cC$, and 
	\item
		satisfies the conditional independencies in $\cI_\cC$.
	\end{itemize}
	Consider the distribution $\tilde{\mu}\in \Dist\bigl(\Asg(\cV\cup \cC')\bigr)$ defined as follows for $b\in \Asg(\cV)$ and $c'\in \Asg(\cC')$:
	\[
		\tilde{\mu}(b,c')
		\ \coloneqq \ 
	  	\mu(b)
	  	\cdot 
	  	\prod_{Y\!\in \cC} \CPTable\bigl( Y\!{=}c'(Y') \mid b_{\Pre(Y)} \bigr).
	\]
	Then, $\tilde{\mu}$ satisfies the above three constraints. 
	Hence, $\tilde{\mu}=\delta$.
	
	For $c \in \Asg(\cC)$, let $c' \in \Asg(\cC')$ denote the corresponding assignment with $c'(Y') = c(Y)$ for $Y\!\in \cC$.
	\begin{align*}
	 	(\gamma \cdot \bP)(c)
	 	\ &= \ 
	 	\delta(c') 
	 	\ = \ 
	 	\tilde{\mu}(c') \\
	 	\ &= \ 
	 	\sum_{\mathclap{ d \in \Asg(\Pre(\cC)) }} \
	 	\mu(d) %
	 	\cdot
	 	\prod_{Y\!\in \cC}
		\underbrace{\CPTable(Y \!{=}c'(Y')\mid d)}_{\CPTable(Y\!{=}c(Y)\mid d)}
		\ = \
		\mu(c)
	 	\ = \ 
	 	\gamma(c).
	\end{align*}
		Hence, $\gamma = \gamma \cdot \bP$ and $\mu \in \semantics{\cB}_\CutC$.
	\qed
\end{proof}

\end{document}